\documentclass{article}

    \usepackage[preprint]{asst/neurips_2022}

\usepackage[utf8]{inputenc} 
\usepackage[T1]{fontenc}    
\usepackage{hyperref}       
\usepackage{url}            
\usepackage{booktabs}       
\usepackage{amsfonts}       
\usepackage{nicefrac}       
\usepackage{microtype}      
\usepackage{xcolor}         

\title{
AgraSSt: 
Approximate Graph Stein Statistics \\
for Interpretable Assessment of
\\
Implicit Graph Generators
}

\author{%
  Wenkai Xu\\
      Department of Statistics\\
University of Oxford \\
\texttt{wenkai.xu@stats.ox.ac.uk} \\
  \And
  Gesine Reinert \\
    Department of Statistics\\
University of Oxford \\
  \texttt{reinert@stats.ox.ac.uk} \\[2ex]
}

\usepackage{graphicx}
\usepackage{subfigure}
\usepackage{booktabs} 
\usepackage{amsmath}
\usepackage{amsfonts}
\usepackage{amsthm}
\usepackage{tabularx}
\usepackage{multirow}
\usepackage{soul}
\usepackage{algorithm}
\usepackage{algorithmic}
\usepackage[algo2e]{algorithm2e}
\usepackage{bbm}
\usepackage{comment}
\usepackage[utf8]{inputenc}

\usepackage{tikz}
\usepackage[english]{babel}
\usepackage{bbold}

\usepackage{amsmath}
\usepackage{amssymb}
\usepackage{mathtools}
\usepackage{amsthm}

\newcommand{\gKSS}{{\rm gKSS}} 

\usepackage[capitalize,noabbrev]{cleveref}

\theoremstyle{plain}
\newtheorem{theorem}{Theorem}[section]
\newtheorem{proposition}[theorem]{Proposition}
\newtheorem{lemma}[theorem]{Lemma}
\newtheorem{corollary}[theorem]{Corollary}
\theoremstyle{definition}
\newtheorem{definition}[theorem]{Definition}

\theoremstyle{remark}
\newtheorem{remark}[theorem]{Remark}

\newcommand{\uk}{\underline{k}} 

\usepackage{todonotes}

\usepackage[utf8]{inputenc}
\usepackage{amsfonts}
\usepackage[english]{babel}
\usepackage{amsthm}
\usepackage{algorithm}
\usepackage{algorithmic}
\usepackage{bbold}
\usepackage{dsfont}

\usepackage{wrapfig}
\usepackage{enumerate}

\newtheorem{Proposition}{Proposition}
\newtheorem{Assumption}{Assumption}

\newcommand{\R}{\mathbb{R}} 
\renewcommand{\H}{\mathcal{H}} 
\newcommand{\G}{\mathcal{G}} 
\newcommand{\X}{\mathcal{X}} 
\newcommand{\A}{\mathcal{A}} 
\renewcommand{\P}{\mathbb{P}} 
\newcommand{\E}{\mathbb{E}} 
\newcommand{\EE}{\mathbb{E}} 
\newcommand{\N}{\mathbb{N}} 
\newcommand{\RR}{\mathbb{R}} 
\newcommand{\PP}{\mathbb{P}}
\newcommand{\T}{\mathcal{A}}
\newcommand{\q}{\widehat{q}}
\newcommand{\bone}{\mathbb{1}}
\newcommand{\I}{\mathbb{1}}
\newcommand{\nullH}{\mathfrak{H}}
\newcommand{\astar}{{a^*}}
\def\given{\typeout{Command 'given' should only be used within bracket command}}
\newcounter{@bracketlevel}
\def\@bracketfactory#1#2#3#4#5#6{
\expandafter\def\csname#1\endcsname##1{%
\addtocounter{@bracketlevel}{1}%
\global\expandafter\let\csname @middummy\alph{@bracketlevel}\endcsname\given%
\global\def\given{\mskip#5\csname#4\endcsname\vert\mskip#6}\csname#4l\endcsname#2##1\csname#4r\endcsname#3%
\global\expandafter\let\expandafter\given\csname @middummy\alph{@bracketlevel}\endcsname
\addtocounter{@bracketlevel}{-1}}%
}
\def\bracketfactory#1#2#3{%
\@bracketfactory{#1}{#2}{#3}{relax}{1mu plus 0.25mu minus 0.25mu}{0.6mu plus 0.15mu minus 0.15mu}
\@bracketfactory{b#1}{#2}{#3}{big}{1mu plus 0.25mu minus 0.25mu}{0.6mu plus 0.15mu minus 0.15mu}
\@bracketfactory{bb#1}{#2}{#3}{Big}{2.4mu plus 0.8mu minus 0.8mu}{1.8mu plus 0.6mu minus 0.6mu}
\@bracketfactory{bbb#1}{#2}{#3}{bigg}{3.2mu plus 1mu minus 1mu}{2.4mu plus 0.75mu minus 0.75mu}
\@bracketfactory{bbbb#1}{#2}{#3}{Bigg}{4mu plus 1mu minus 1mu}{3mu plus 0.75mu minus 0.75mu}
}
\bracketfactory{clc}{\lbrace}{\rbrace}
\bracketfactory{clr}{(}{)}
\bracketfactory{cls}{[}{]}
\bracketfactory{abs}{\lvert}{\rvert}
\bracketfactory{norm}{\Vert}{\Vert}
\bracketfactory{floor}{\lfloor}{\rfloor}
\bracketfactory{ceil}{\lceil}{\rceil}
\bracketfactory{angle}{\langle}{\rangle}
\newcommand{\AgraSSt}{{\rm AgraSSt}}
\newcommand{\wk}[1]{\textcolor{black} {#1}}
\newcommand{\gr}[1]{\textcolor{black}{#1}}

\begin{document}

\maketitle

\begin{abstract}
We propose and analyse 
a novel statistical procedure, coined {\it AgraSSt}, to assess the quality of graph generators which may not be available in explicit forms.  
In particular, AgraSSt can be used  
to determine whether a \emph{learned} 
graph generating process is 
capable of generating graphs 
which resemble a given input graph. 
Inspired by Stein operators for random graphs, the key idea of AgraSSt is the construction of a kernel discrepancy based on an 
operator obtained from the graph generator.
AgraSSt can 
provide interpretable criticism{s} for a graph generator training procedure and help 
identify reliable sample  
batches for downstream tasks.
We 
give theoretical guarantees for a broad class of 
random graph models. 
We provide empirical results on both synthetic input graphs with  \emph{known} graph generation procedures,
and 
real-world 
input graphs
that the state-of-the-art ({deep}) generative models for graphs are  \emph{trained} on.

\end{abstract}


\medskip 
{\bf MSC2020 Subject Classification.} 60E05, 62E17, 60B20, 05C80

\section{Introduction}
\label{sec:intro}

Generative models for graphs have received increasing attention in  the statistics and machine learning communities. 
Recently, deep neural networks have been utilised to learn rich representations from graph structures and generate graphs 
\citep{dai2020scalable, li2018learning, liao2019efficient, you2018graphrnn}. 
However, due to the 
often opaque deep learning procedures, these deep generative models 
are usually implicit, which hinders theoretical analysis to assess how
close the generated samples are in their distributional properties to the graph distribution they are meant to be sampled from.

Learning parametric models requires an explicit pre-specified probability distribution class, {and}
the learned parameters can be used for model assessment. However, parametric models may only capture a fraction of the graph features and have restricted modelling power. 
Due to dependency between edges, parameter estimation can be inconsistent \citep{shalizi2013consistency} even for well-specified models, and {may} lead to wrong conclusions in model assessment{s}.
Although recent advances in deep generative models for graphs may surpass some of the above-mentioned issues by learning rich graph representations, methods to assess the quality of such implicit graph generators are lacking. 
In principle, nonparametric
hypothesis {tests} 
can be useful 
to assess complex models, {such as by the popular kernel-based tests procedures}
that utilise functions in {a} reproducing kernel Hilbert space (RKHS) \citep{RKHSbook}. 
When the models are described in the form of explicit probabilities, goodness-of-fit tests \citep{chwialkowski2016kernel, liu2016kernelized} may apply; however, goodness-of-fit testing procedures are generally not applicable for implicit models. Instead, {if a large 
{set of} sample{s} from the target distribution is observed,} one may generate samples from the implicit model and perform a two-sample test {such as a maximum mean discrepancy (MMD) test} \citep{gretton2007kernel} for model assessment 
\citep{jitkrittum2017linear, xu2020stein, xu2021interpretable}. 



\begin{figure*}[t]
    \centering
    \includegraphics[width=1.\textwidth, 
    ]{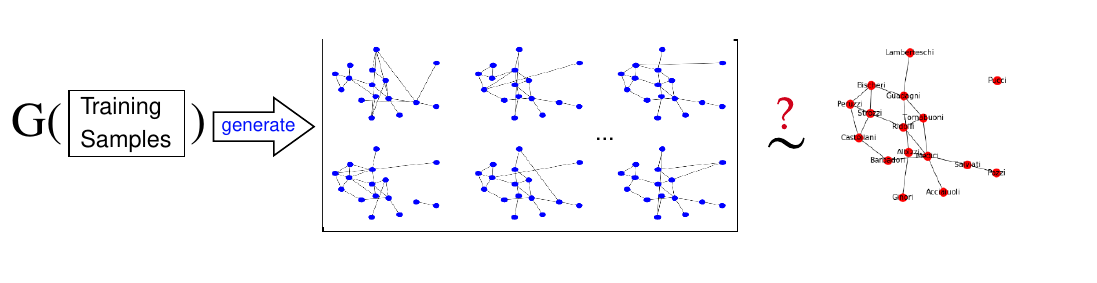}
    \vspace{-.8cm}
    \caption{Illustration for the assessment {task}:
    a graph generator $G$, which is learned from training samples, {generates  a set of network samples} (vertices in blue), 
    and is assessed against the {target} graph, which here is Padgett's Florentine marriage network (with vertices in red labelled as family names).
    }
    \label{fig:florentine}
\end{figure*}

However, in real-world applications, {often only a single graph from the target distribution is observed}
\citep{bresler2018optimal, reinert2019approximating}. 
{In this case,} the MMD 
methods for model assessment via
two-sample testing procedures 
cannot be
used
to assess implicit graph generators.
To the best of our knowledge, {beyond 
Monte Carlo tests 
based on a simple test statistic {with often poor power}, 
no principled test}
{is available} for assessing the quality of implicit graph generators. 
%
%
In Figure.\ref{fig:florentine}, the task is illustrated: we are given a graph generator $G$ which is  learned from a set of training samples
and can generate samples of user-defined size. {The task is} to assess whether $G$
can generate samples
from the same distribution that generates the observed graph, 
{for example, Padgett's} Florentine marriage network \citep{padgett1993robust}.
In Figure.\ref{fig:florentine}, $G$  is a  Cross-Entropy Low-rank Logit (CELL) model \citep{pmlr-v119-rendsburg20a} trained on Padgett's Florentine marriage network, {but any graph generator could be used.}
{This paper makes three main contributions.} 
(1)
{We {introduce}
\underline{A}pproximate \underline{gra}ph \underline{S}tein \underline{St}atistics (\AgraSSt) in Section \ref{sec:approx_ergm_stein}, which
{opens up a \emph{principled way} to} understand implicit graph generators.
\AgraSSt{} is a variant of a kernel Stein discrepancy
based on an empirical Stein operator for  conditional distributions of general random graph models. {The} 
testing procedure is inspired 
{by gKSS, a} goodness-of-fit testing {procedure} for explicit {exponential random graph models (ERGMs)} \citep{xu2021stein}.
(2) {We provide theoretical guarantees for \AgraSSt{} (\cref{sec:approx_ergm_stein}).} 
(3)
We propose interpretable model criticisms when there is a {model} misfit
and identify representative synthetic sample batches when the model is a good fit (\cref{sec:application}).
%
%

{{Further, i}n Section~\ref{sec:background} we review gKSS.} 
{We provide} 
empirical results in Section~\ref{sec:exp} and {a discussion}
in Section~\ref{sec:future}.
Proof details, {more background,} 
theoretical and empirical results, 
and 
implementation details 
{are found} in {the}  Supplementary Information (SI) {text}. 
The code for the experiments is
available at \url{https://github.com/wenkaixl/agrasst}.

\section{Background{: graph kernel Stein statistics}}\label{sec:background} 
{In \citep{xu2021stein} a goodness-of-fit testing procedure called gKSS for exponential random graph models is introduced, which assumes that only a single network may be observed in the sample. As AgraSSt is inspired by gKSS, and gKSS serves as comparison method when the observed network is known to be generated from an exponential random graph model {(ERGM)}, we briefly review gKSS. The notation introduced in this section is used throughout the paper.}
The class of 
ERGMs is  
a rich model class {which includes}  
Bernoulli random graphs
and 
{which is} extensively used for social network analysis \citep{wasserman1994social, holland1981exponential, frank1986markov}.
ERGMs model random graphs via a Gibbs measure with respect to a chosen set of network statistics, such as  number of edges, 2-stars, and triangles.
Denote by  $\G^{lab}_n$, the set of vertex-labeled graphs on $n$ vertices, with 
{$N =n(n-1)/2 $ possible undirected edges.}
{Encode} $x \in \G^{lab}_n$ by an ordered collection of $\{0,1\}$-valued variables $x = (x^{(ij)})_{1 \le  i < j \le n} \in \{0,1\}^N$ {where} $x^{(ij)}=1$ {if and only if}
there is an edge between $i$ and $j$. We denote an (ordered) {vertex-pair index $s=(i,j)$}  by $s\in [N]:=\{1, \ldots, N\}$.
For a {collection of graphs}  
$H_1, \ldots, H_k$ on at most $n$ vertices,  with $H_1$ denoting a single edge, let $t_\ell$ denote the number of counts of $H_\ell$ in the observed subgraph (possibly scaled; for details including a precise definition, \cref{def:ergm}, see SI\,\ref{sigauss}), 
For  $\beta = (\beta_1, \dots, \beta_k)^{\top} {\in \R^k}$ 
and
$t(x) =(t_1(x),\dots,t_k(x))^{\top} \in \R^k$ 
we say that 
$X\in \G^{lab}_n$ follows  the exponential random graph model  $X\sim \operatorname{ERGM}(\beta, t)$ if for  $\forall x\in \G^{lab}_n$,
\begin{equation}\label{eq:ergm}
    q(X = x) = \frac{1}{\kappa_n(\beta)}\exp{\left(\sum_{l=1}^{k} \beta_l t_l(x) \right)}.
\end{equation}
Here $\kappa_n(\beta)$ is a  normalisation constant. {In this model, $t_\ell (x), \ell=1, \ldots, k$, are sufficient statistics.} 

Parameter estimation $\hat \beta_l$ for $\beta_l$ is \emph{only} possible when {the graphs $H_2, \ldots, H_k$} 
are specified a priori; see SI.\ref{app:param-estimation} for estimation details.
In modern graph learning procedures, e.g. deep generative learning, the 
sufficient statistics $t_l$ of Eq.\eqref{eq:ergm}  may not be obtained explicitly.

{In \citep{reinert2019approximating} the exponential random graph distribution in Eq.\,\ref{eq:ergm} is characterised by a so-called {\it Stein operator}, as follows. Let} 
$e_s \in  \{0,1\}^N$ be a vector with $1$ in coordinate {$s$} and 0 in all others; 
$x^{(s,1)} = x + (1-x_s) e_s$ has the $s$-entry  replaced of $x$ by the value 1, and $x^{(s,0)} = x - x_s e_s$ has  the $s$-entry  of $x$ replaced by the value 0; moreover,
${x}_{- s}$ is the set of edge indicators with entry $s$ removed. For a  function
$h: \{0,1\}^N \rightarrow \R$, let
$\Delta_s h(x) = h( x^{(s,1)}) - h(x^{(s,0)}).$ Set 
$ q_X(x^{(s,1)} | {x_{-s}} ) = \P ( X^s=1| {X_{-s} = x_{-s}}).$ Define the operator
\begin{eqnarray} 
\A_{\beta, {t}} f(x)
	=  \frac{1}{N} \sum_{s\in[N]} \A^{(s)}_q f(x),
	\label{eq:ergm_stein} 
\quad
    \A^{(s)}_q f(x) 
    = q(x^{(s,1)}|{x_{-s}} ) \Delta_s f(x) 
     + \left( f(x^{(s,0)}) - f(x)\right).
\end{eqnarray}
Then under mild conditions \cite{reinert2019approximating} show that if 
$\E_p[{\A}_{\beta, t}  f] = 0$ for all smooth test functions $f$, then $p$ {must be} the distribution of ERGM$(\beta, t)$. Thus, this operator characterises  ERGM$(\beta, t)$. 
For the derivation of AgraSSt it is of interest to see how this operator is obtained. It is indeed the generator of a  so-called {\it Glauber} Markov chain on $ \G^{lab}_n$ with  transition probabilities 
$$\P (x \rightarrow x^{(s,1)} ) = N^{-1}
-
 \P( x \rightarrow  x^{(s,0)}) = N^{-1} 
 q(x^{(s,1)} | x_{-s}).$$
}

{Further, as  
$$\Delta_s t_\ell (x)
= t_\ell(x^{(s,1)}) - t_\ell(x^{(s,0)}),$$
cancelling out common factors, \begin{eqnarray}
q(x^{(s,1)}| {x_{-s}} ) &= & \exp\left\{\sum_{\ell=1}^L \beta_\ell \gr{\Delta_s t(x)}\right\} \times   
\left( \exp\left\{ \sum_{\ell=1}^L \beta_\ell \gr{\Delta_s t(x)}\right\} +1 \right)^{-1}  \nonumber\\
&=&
q(x^{{(}s,1{)}} |\gr{\Delta_s t(x)} )
\label{eq:condexp}
\end{eqnarray}
depends only on $\Delta_s t(x)$.} 
With the ERGM Stein operator in Eq.\eqref{eq:ergm_stein} and a rich-enough RKHS test function class 
$\H$, \citet{xu2021stein} propose a graph kernel Stein statistics (gKSS) to perform goodness-of-fit testing on {an} \emph{explicit} ERGM when a single network sample is observed.
With the summand components in Eq.\eqref{eq:ergm_stein}, the Stein operator 
can be seen as taking expectation over {vertex-pair} variables $S\in[N]$ with uniform probability $\P(S=s)\equiv N^{-1} $
independently of $x$,  namely 
\begin{align*}
\A_q f(x) &= \sum_{s\in [N]} \P(S=s)\A^{(s)}_q f(x) =: {\E_S [ \A^{(S)}_q f(x)]} . 
\end{align*}
For a fixed graph  $x$, 
gKSS is defined as
\begin{align}\label{eq:gkss}
    \operatorname{gKSS}(q;x) 
    & = \sup_{\|f\|_{\H}\leq 1} \Big|\E_S[\A^{(S)}_q f(x)] \Big|,
\end{align}
where the function $f$ is chosen to best distinguish $q$ from $x$.
For {an} RKHS $\H$ associated with kernel $K$, 
{by the reproducing property of $\H$,} 
the squared version of gKSS admits a quadratic form {representation} 
$\operatorname{gKSS}^2(q;x) = \Big\langle \E_S[\A^{(S)}_q K(x,\cdot)],\E_S[\A^{(S)}_q K(x,\cdot)]  \Big\rangle,$
which can be computed readily.
{More background can be found in \cref{sec:ksd}.}

\section{AgraSSt: Approximate Graph Stein Statistic}\label{sec:approx_ergm_stein}

An implicit graph generator may not admit a  probability distribution in the form of Eq.\eqref{eq:ergm}; however, the idea of constructing Stein operators based on Glauber dynamics using conditional probability distribution for ERGM is 
inspiring.
Here we propose  Stein operators for conditional graph distributions, to facilitate an (approximate) characterisation 
for implicit random graph models. 

\subsection{Stein operators for conditional graph distributions}
Let $q(x) =\mathbb{P}(X=x)$ be any distribution with support $\G^{lab}_n$.
Let $t(x)$ denote a statistic on graphs which takes on finitely many values $\uk$ and let 
$q_{\uk} ( x) = \mathbb{P}(X=x | t(x) = \uk)$.  We assume that $q_{\uk} ( x) > 0$ for all $\uk$ under consideration. 
{For a generic outcome we write $q_t$.} 
{Inspired by \eqref{eq:condexp}, we} introduce a Markov chain on $\G^{lab}_n$
which  transitions from $x$ to $x^{s,1}$ with probability 
\begin{equation}\label{eq:cond_prob}
    q(x^{s,1} | t(x_s) ) = \mathbb{P} (X^s = 1 | {\Delta_s t(x)}) =: q_t (x^{s,1}) ,
\end{equation} 
and which transitions from $x$ to 
$x^{s,0}$ with probability $q(x^{s,0} | {\Delta_s t(x)} ) 
= 1- q_t (x^{s,1}) ;$ {no other transitions occur.}
Let 
\begin{align}
\mathcal{A}_{q,t}^{(s)} f(x) {:= \mathcal{A}_{q,\Delta_s t(x)}^{(s)} f(x) }
 :=  q_t(x^{s,1})f( x^{(s,1)})  + q_t(x^{s,0})f(x^{(s,0)})  - f(x).
\label{eq:cond_stein}
\end{align}
For an ERGM,  $t(x)$ {could be taken} as a sufficient vector of statistics\footnote{When conditioning on the sufficient statistics, for ERGMs  the resulting Stein operator 
allows
to establish elegant approximation results \citep{bresler2019stein, reinert2019approximating}.}, but {here we do not assume  a  parametric network model $q(x)$, and}
$t(x)$ does not have to be sufficient statistics for $q(x)$.

Recall that an operator is a Stein operator for a distribution $\mu$ if its expectation under $\mu$ is zero. {The following result
provides a theoretical foundation for \AgraSSt{}
and is proven in SI.\ref{app:proof}.} 
{ 
\begin{lemma}\label{Stein operators}
{In this setting, $\mathcal{A}_{q,{\Delta_s t(x)}={\uk}}^{(s)}$ is a Stein operator for the conditional distribution of $X$ given ${\Delta_s t}(X)= \uk$, and 
$ \sum_s \mathcal{A}_{q,{\Delta_s t(x)} = {\uk}}^{(s)}$ is a Stein operator for 
the conditional distribution of $X$ given ${\Delta_s t(X)}=\uk$. 
}
\end{lemma} 
In particular, 
$\mathbb{E}{_{q_{t=\uk}}} \left[  \mathcal{A}_{q, {\Delta_s t(x)} ={\uk}}^{(s)} 
\right] =0.$ 
Intuitively, if $\tilde{X}$ has distribution which is close to that of $X$, then with $\tilde{Y}_{{\Delta_s t(x)}={\uk}}$ denoting the corresponding random graph with distribution that of ${\tilde X}$ given ${\Delta_s t(X)}= \uk,$ it should hold that 
$\mathbb{E} \left[  \mathcal{A}_{q,t={\uk}}^{(s)} ({\tilde Y}_{{\Delta_s t(Y)}={\uk}}) \right] \approx 0$. 
In this way the Stein operator in Eq.\eqref{eq:cond_stein} can be used to assess the similarity between distributions. 
}

\subsection{Approximate Stein operators}

For implicit models and graph generators $G$, the Stein operator $\mathcal{A}_{q,t}^{(s)}$ in Eq.\eqref{eq:cond_stein}
cannot be obtained without explicit knowledge of 
$q_t(x^{s,1})$. However, {given a large number for samples from the graph generator $G$}, 
the conditional edge probabilities 
$q_t(x^{s,1})$ can be estimated. {Here we 
denote by  
$\q_t(x^{s,1})$
{an estimate of} 
$q_t(x^{s,1})$;
{some estimators will be suggested in \cref{sec:estimation}}.} 


AgraSSt {performs} 
model assessment using {an operator which approximates the} 
Stein operator
$\mathcal{A}_{q,t}^{(s)}$. 
{We define t}he approximate Stein operator for {the} conditional random graph by 
\begin{align}
\mathcal{A}_{\q,t}^{(s)} f(x) 
 =  \q_t(x^{s,1})f(x^{(s,1)})  + \q_t(x^{s,0})f( x^{(s,0)})  - f(x).
\label{eq:approx_cond_stein}
\end{align}
The 
{vertex-pair} averaged {approximate} Stein operator is
\vspace{-2mm} 
\begin{equation}\label{eq:approx_stein}
\A_{\widehat q, t} f(x) = \frac{1}{N} \sum_{s \in [N]} \mathcal{A}_{\q,t}^{(s)} f(x).
\end{equation}
\vspace{-2mm} 
\subsection{Estimation with chosen 
statistics
on graphs
}\label{sec:estimation}

Using the Stein operator for conditional graph distributions, we can obtain the approximate Stein operators Eq.\eqref{eq:approx_cond_stein} and \eqref{eq:approx_stein} for an implicit graph generator $G$ by estimating 
$q_t(x^{s,1})$. 
Here $t(x)$ are user-defined statistics.
{In principle, any multivariate statistic $t(x)$ can be used in this formalism. However, estimating the conditional probabilities using relative frequencies {can} 
be computationally prohibitive when the graphs are very large and specific frequencies are rarely observed. Instead, here we consider simple summary statistics, such as} edge density, degree statistics or {the} number of neighbours connected to both vertices of $s$.  
The estimation procedure is presented in Algorithm \ref{alg:est_conditional}.

\begin{algorithm}[t]
   \caption{Estimating {the} conditional probability 
   }
   \label{alg:est_conditional}
\begin{algorithmic}[1]
\renewcommand{\algorithmicrequire}{\textbf{Input:}}
\renewcommand{\algorithmicensure}{\textbf{Procedure:}}
\REQUIRE
Graph generator $G$; statistics $t(x)$;
\ENSURE~~\\
\STATE Generate samples $\{x_1,\dots,x_L\}$ from $G$.
\STATE For ${s\in[N], i\in[n]}$,
{let $n_{s, k}$ the number of 
graphs {$x_l, l \in [L], $} in which $s$ is present and 
${\Delta_s t(x)}=k$.} 
\STATE Estimate the conditional probability of {the edge $s$ being present conditional on ${\Delta_s t(x)}=k$ by an estimator
{${\widehat{g_t}}(s;k) $}, using a look-up table or smoothing.} 
\renewcommand{\algorithmicensure}{\textbf{Output:}}
\ENSURE
{${\widehat{g_t}}(s;k) $}
that estimates 
$q(x^{(s)} {=1} |{\Delta_s t(x)}=k)$.
\end{algorithmic}
\end{algorithm}

To estimate 
$q_t(x^{s,1})$ in Step~3 of Algorithm \ref{alg:est_conditional}, if the underlying graph has exchangeable edge indicators  then
$q_t(x^{s,1})$ does not depend on the choice of vertex-pair $s$, 
an intuitive way is to use a \emph{lookup table}. 
{If} $t(x)$ is a possibly multivariate statistic with a discrete number of outcomes, count {$n(\uk, s)$, the number of times that vertex-pair $s$ is present in the simulated graphs and
} ${\Delta_s t(x)} = \uk$;  {set $n(\uk) = \sum_s n( \uk, s),$} and 
$N_{\uk} = \sum_{i =1}^L \sum_s {\mathbb{1}} ({\Delta_s t(x)}= k)$;
{if $t = \uk$ then estimate $\q_t(x^{s,1})$ by} 
\vspace{-2mm} 
\begin{align}\label{gk}
    \hat g_t(\uk) = \frac{n_k}{N_{\uk}}{{\mathbb{1}}(N_{\uk} \ge 1)}.
\end{align}
{If} $k_{min}, k_{max}$ denote the minimum and maximum values of statistics from simulated graph samples, {then for $\uk$ outside this set,} 
the lookup table {estimator Eq.\eqref{gk}} {{is set to estimate} 
$\hat g(\uk)=0 $ .

If the underlying graph cannot be assumed to have exchangeable edge indicators or if the statistic 
$t$ is high dimensional, then any particular $n(\uk, s)$ may not be observed very often. In such a situation we
can learn $\hat g_t(s;k)$ using kernel ridge regression so that the conditional probabilities for similar ${\Delta_s t(x)}$ are predicted in a smooth manner. 
%
Estimating $q(x^s|{\Delta_s t(x)} \leq k)$ instead of  $ q(x^s|{\Delta_s t(x)}= k)$
{may} provide an alternative, smoother estimate for the conditional probabilities. 

The next result 
 shows that the approximate Stein operator achieves 
the Stein identity 
asymptotically.
For this result, which is proved in SI.\ref{app:proof}, we use the notation
$\| \Delta f\| = \sup_{s \in [N], x} \|\Delta_s f (x)\|.$

\begin{theorem} \label{th:consist}
Assume 
$\q_t(x^{s,1})$
is a consistent estimator for 
$q_t(x^{s,1})$ 
as $L \rightarrow \infty$. 
Then {for any 
$f$ 
such that
 $ || \Delta f || < \infty $}  we have $\E_q[\A_{ \q, t}f(x)] \to \E_q[\A_{q,t}f(x) ] =0$ as {$L \rightarrow \infty$.} 
\end{theorem}
Section \ref{sigauss} in SI.\ref{app:proof} 
provides refined results for ERGMs, including a Gaussian approximation. 

\subsection{
AgraSSt for implicit graph generators
}
The estimated conditional probabilities 
give an approximate Stein operator for Eq.\eqref{eq:cond_stein}. 
With the appropriately defined Stein operator from {an} implicit model {given} in Eq.\eqref{eq:approx_cond_stein}, we 
can define \AgraSSt, 
a kernel-based statistic analogous to gKSS in Eq.\eqref{eq:gkss}, as
$\operatorname{AgraSSt}(\widehat q, t;x) 
 = \sup_{\|f\|_{\H}\leq 1} \Big|N^{-1}\sum_s \A^{(s)}_{\widehat q,t} f(x)\Big|.$
%
{In SI.\ref{app:proof} we prove the following result.} 

\begin{theorem}\label{consistenterop}
{If the graph is edge-exchangeable, then}
{as $L \rightarrow \infty$, }  ${\AgraSSt}^2 (\q,t;x) $
is a consistent estimator of 
\begin{equation}
{\rm{gKSS}}^2 (q; x) = N^{-2}  \sum_{s, s'\in [N]} 
\left\langle \A^{(s)}_{ q}  K(x,\cdot), \A^{(s')}_{ q}  K(\cdot,x)\right\rangle_{\H}.
    \label{eq:gkss_quadratic}
\end{equation}
\end{theorem}

\paragraph{Re-sampling Stein statistic}
A computationally efficient operator for large $N$ 
can be derived via re-sampling
vertex-pairs $s$,
which creates a randomised operator. 
Let $B$ be the {fixed} size 
to be re-sampled. The re-sampled 
operator is
${\widehat\A_{\q,t}^B} f(x) = \frac{1}B \sum_{b\in [B]} \A^{(s_b)}_{\q,t} f(x),$
where 
$s_b$ are vertex-pair samples from $\{1, \ldots, N\}$, {chosen uniformly with replacement, independent of each other and of $x$.}
The  expectation of  $ {\widehat\A_{\q,t}^B} f(x)$ with respect to re-sampling is 
$
\E_B [ {\widehat\A_{\q,t}^B} f(x)  ]
{ = \E_S [ \A_{\q,t}^{(S)} f(x)]} 
= \A_{\q,t} f(x). 
$

The corresponding re-sampled AgraSSt is
$ \widehat{\operatorname{AgraSSt}}(\widehat q, t;x) = \sup_{\|f\|_{\H}\leq 1} \Big|\frac{1}{B}\sum_{b\in [B]}\A^{(s_b)}_{\widehat q,t} f(x) \Big|.$

{Similar to Eq.\eqref{eq:gkss_quadratic}, the squared version of $\widehat{\operatorname{AgraSSt}}$ admits {a representation} in a quadratic form,
\begin{align}\label{eq:AgraSSt_resample_quadratic_form}
    \widehat{\operatorname{AgraSSt}}^2({\widehat q,t};x) = B^{-2
}  \sum_{b, b'\in [B]} \widehat h_x(s_b, s_{b'}),
\end{align}
where 
$\widehat h_x(s, s') = \left\langle \A^{(s)}_{\widehat q,t} K(x,\cdot), \A^{(s')}_{\widehat q,t} K(\cdot,x)\right\rangle_{\H}.$
}
%
{We note that t}he randomised operator obtained via re-sampling is a form of stochastic Stein discrepancy {as introduced in} \citep{gorham2020stochastic}.


{For fixed $x$, {under mild conditions}  the consistency of $ \widehat{\operatorname{AgraSSt}^2}({\widehat q,t};x)$ as $B \to \infty$ is ensured by the following normal approximation, which follows 
from Proposition 2 in \citep{, xu2021stein}.
\begin{proposition}\label{Bconsistency} Assume that $\widehat h_x(s, s')$ in \eqref{eq:AgraSSt_resample_quadratic_form} is bounded and that 
$ \widehat{\operatorname{AgraSSt}}({\widehat q,t};x) $ has 
non-zero 
variance $\sigma^2$. Let $Z$ be a normal variable with mean 
${\operatorname{AgraSSt}}({\widehat q,t};x) $ and variance $\sigma^2$. Then there exists an explicitly computable constant $C>0$ such that for all 3 times continuously differentiable functions $g$ with bounded derivatives up to order 3, 
$$ \mathbb{E} [ g( \widehat{\operatorname{AgraSSt}}({\widehat q,t};x) ) - g(Z) ] \le \frac{C}{B}.$$ 
\end{proposition}
}

\section{Applications of AgraSSt}\label{sec:application}

\subsection{Assessing graph generators}\label{sec:assessing}

\AgraSSt{} measures {the} distributional difference between {the underlying distribution of} an implicit graph generator $G$ and an observed graph $x$, which is useful to assess the {quality of the generator $G$}.
The  
{hypothesis testing}
procedure {for the null hypothesis that the observed graph $x$ comes from the same distribution that generates the samples, against the general alternative,} is shown in Algorithm \ref{alg:AgraSSt}.
{We emphasise two features of this procedure.}
%
{Firstly, for} a given generator $G$, \AgraSSt{} directly assesses the quality of the implicit model represented via samples from $G$. {Secondly, the generator} 
$G$ can be {trained} on the observed graph $x$, for example through a deep neural network generator. 
%
By learning a deep neural network generator with training samples from the same distribution that generate $x$, \AgraSSt{} can assess the quality of the training procedure, i.e. whether the deep neural network is capable of learning the desired distributions. Additional details are discussed in SI.\ref{app:illustration}.
%


\subsection{Interpreting trained graph generators}
If {the procedure in}  Algorithm \ref{alg:AgraSSt} rejects {the null hypothesis}, the generator may not be {suitable}
for generating samples from the distribution that generates the one observed graph. 
Hence, understanding where the misfit comes from can be very useful, especially for models trained from black-box deep neural networks.
\AgraSSt{} 
provides an interpretable model criticism by comparing the learned $\q_t(x^{s,1})$ with the  underlying $q_t(x^{s,1})$ {when available, such as in synthetic experiments from a specified ERGM}. {Such an  interpretation can be}  
also useful to re-calibrate training procedures.

\subsection{Identifying reliable graph samples}\label{sec:reliable-sample}


If {the procedure in} Algorithm \ref{alg:AgraSSt} does not reject {the null hypothesis}, there is 
{not enough evidence}
to reject the hypothesis that
the generator is capable of generating graphs that resembles the observed graph. 
{If a generator $G$ has passed this hurdle then it can be recommended for generating graph samples of the desired type. \AgraSSt{}  can also be put to use for the task of sample batch selection.} 
In real scientific studies, only a small batch of 
representative graph samples may need to be generated for downstream tasks {such as}
privacy-preserving {methods}
where users only access a small number of {graph}
data, or {a} randomised experimental design for community interaction.
To quantify the quality of sample batches
via 
$p$-values : 
(1) generate a sample batch {of size $m$, say}; (2) perform {the} steps in \textbf{Procedure} from Algorithm \ref{alg:AgraSSt}; (3) 
compute {the} $p$-value $m^{-1} {\sum_{i=1}^m \mathbb{1}(\tau>\tau_i)}$, with $\tau$ as in step 3 and  $\{\tau_1,\dots,\tau_m\}$ as in step 5.
If the $p$-value is smaller {than a pre-specified threshold}, generate another  sample batch; otherwise accept the current sample batch.

\begin{algorithm}[t]
   \caption{Assessment procedures for graph generators}
   \label{alg:AgraSSt}
\begin{algorithmic}[1]
\renewcommand{\algorithmicrequire}{\textbf{Input:}}
\renewcommand{\algorithmicensure}{\textbf{Objective:}}
\REQUIRE
    Observed graph $x$; graph generator $G$ and generated sample size $L$; 
    estimation statistics $t$; RKHS kernel $K$; re-sampling size $B$; {number of simulated graphs $m$}; confidence level $\alpha$;
\renewcommand{\algorithmicensure}{\textbf{Procedure:}}
\ENSURE~~\\
\STATE Estimate $\widehat q(x^{(s)}|{\Delta_s t(x)})$ based on Algorithm \ref{alg:est_conditional}.
\STATE Uniformly generate re-sampling index
$\{s_{1},\dots,s_{B}\}$ from $[N]$ with replacement.
\STATE Compute $\tau =\widehat{\operatorname{AgraSSt}}^2(\widehat q;x)$ in Eq.\eqref{eq:AgraSSt_resample_quadratic_form}. 
\STATE Simulate $\{z'_1,\dots,z'_m\}$ from $G$.
\STATE Compute $\tau_i =\widehat{\operatorname{AgraSSt}}^2(\widehat q;z'_i)$ in Eq.\eqref{eq:AgraSSt_resample_quadratic_form}.
\STATE {Estimate} {empirical}  
quantile $\gamma_{1-\alpha}$ via $\{\tau_1,\dots,\tau_m \}$.
\renewcommand{\algorithmicrequire}{\textbf{Output:}}
\REQUIRE
Reject the null if $\tau > \gamma_{1-\alpha}$; otherwise do not reject.
\end{algorithmic}
\end{algorithm}

\section{Empirical {results}
}\label{sec:exp}

We first 
{illustrate the performance of \AgraSSt{}} 
on synthetic data, where the null distribution is 
\textit{known} {and we have control of the set-up; in particular we can illustrate the use of \AgraSSt{} for interpretable model criticism.} Then we 
{show the performance of \AgraSSt{}} 
on 
{a} real-world data application to assess 
graph generator{s} trained via various deep generative models.

\subsection{Synthetic experiments}\label{sec:synthetic_exp}
{Only few  competing approaches are available for our task and many of them are devised specifically for ERGMs. Hence here we use an ERGM, namely}
the Edge-2Star-Triangle (E2ST) model with
\begin{equation}\label{eq:e2st_model}
    q(x) \propto \exp \left(\beta_1 E_d(x) + \beta_2 S_2(x) + \beta_3 T_r(x) \right),
\end{equation}
where $E_d(x)$ denotes the number of edges of $x$, $S_2(x)$ denotes the 2-Star statistics and $T_r(x)$ denotes the triangle statistics.  
{Here,} $\beta = (-2.00, 0.00, 0.01)$ is chosen as the null model, 
while alternative {models are}  constructed by perturbing 
the coefficient 
$\beta_2$ {as in}
\citet{yang2018goodness, xu2021stein}. 
%
{This 
particular ERGM  {is} 
chosen {because}
it is the currently most complex ERGM for which a thorough theoretical analysis for parameter estimation is available, see \citet{mukherjee2013statistics}\footnote{The model also satisfies conditions in Theorem 1.7 in \citet{reinert2019approximating} where theoretical properties are  studied.
}.} 


\subsubsection{{Related} 
approaches {for comparisons}} \label{subsec:related}

To assess the performance of \AgraSSt{}, we
consider {the following}
existing test statistics {which are either} tailored or modified 
to perform assessment for implicit graph generators:
\textbf{Deg} {is a} degree-based statistics 
for
goodness-of-fit test of exchangeable random graphs \citep{ouadah2020degree} {based on the estimated} 
variance of the degree distribution. 
The statistics can be obtained from empirical degrees from samples generated from the implicit model.  
\textbf{TV\_deg} {denotes the} Total-Variation (TV)  distance between degree distributions. \citet{hunter2008goodness} 
 proposes a simulation-based approach to construct graphical goodness-of-fit tests. \citet{xu2021stein}
quantifies {this approach using the}
total-variation distance between the distributions of chosen network statistics; see SI.\ref{app:TV} for details.
\textbf{MDdeg} {is the} Mahalanobis distance between degree distributions \citep{lospinoso2019goodness}.

In the synthetic experiment where the parametric model is known (Eq.\eqref{eq:e2st_model}), {the}  
coefficients $\beta$
{are} estimated from generated samples for model assessment {to provide our} baseline approach,
denoted by \textbf{Param}.
Details can be found in SI.\ref{app:param-estimation}. 
Knowing the explicit null model in the synthetic setting, we 
can also compute gKSS in Eq.\eqref{eq:gkss}, 
denoted by \textbf{Exact} as our benchmark. {The} Weisfeiler-Lehman graph kernel  \citep{shervashidze2011weisfeiler} with height parameter 3 is used for kernel-based approaches.

\subsubsection{Simulation results} \label{subsec:simresults}

The rejection rates for various 
settings are shown in Figure.\ref{fig:e2st}. 
As the null model is relatively sparse
(edge density $11.2\%$), the sparser alternatives are much harder to distinguish while the denser ones are easier problems. 
%
In Figure.\ref{fig:e2st_validation}, we compare the AgraSSt procedure with {the approaches from 
\cref{subsec:related}.} 
From the results, we see that \textbf{AgraSSt} performs competitively {to} the benchmark \textbf{Exact}, {which is only available when the model is known} explicitly, and outperforms other assessment procedures for implicit models.
\textbf{TV\_deg} {is} slightly less powerful than \textbf{AgraSSt} {but outperforms} 
\textbf{Deg}.
Due to the small perturbation of the model parameter for the alternative, creating a hard problem for  sparser alternatives,  \textbf{MDdeg} and \textbf{Param} 
have {a much} lower rejection rate {and thus} less powerful for sparser alternatives.  
In Figure.\ref{fig:e2st_statistics}, we compare the performance of \textbf{AgraSSt} with different estimation methods for $\q_t(x^{s,1})$. The degree  $deg(k)$ of a vertex $k$ is calculated  {excluding the 
{vertex-pair $s=(i,j)$}};
\textbf{Sum\_deg}:  {for $s=(i,j)$} we set}
$t = deg(i) + deg(j)$;
\textbf{Cum\_deg}: the cumulative distribution function of sum of degrees are used;
\textbf{Bi\_deg}: {for $s=(i,j)$} {the 2-dimensional vector} ${t=(deg(i),deg(j))}$ {is} used;
\textbf{Edges}: $t$ 
is the edge density {after} removing vertex-pair $s$.
From the results, we see that {$\textbf{Edges}$}
{outperforms the} 
other estimates, which echos the theoretical results shown in Theorem \ref{th:normalapprox} in SI, {as} the coefficient $\beta$ for E2ST satisfies {its assumptions}. We also see that using both vertex degrees as $2$d vector predicts substantially better than using predictors based on sum of degrees of two vertices.
In 
Figure.\ref{fig:e2st_resampling}, the comparison with re-sampling is shown. {
With}
increase in re-sampling size, {the power of}  AgraSSt {increases}.

\begin{figure*}[t!]
    \centering
    \subfigure[Different assessment approaches
    ]{
    \includegraphics[width=0.32\textwidth]{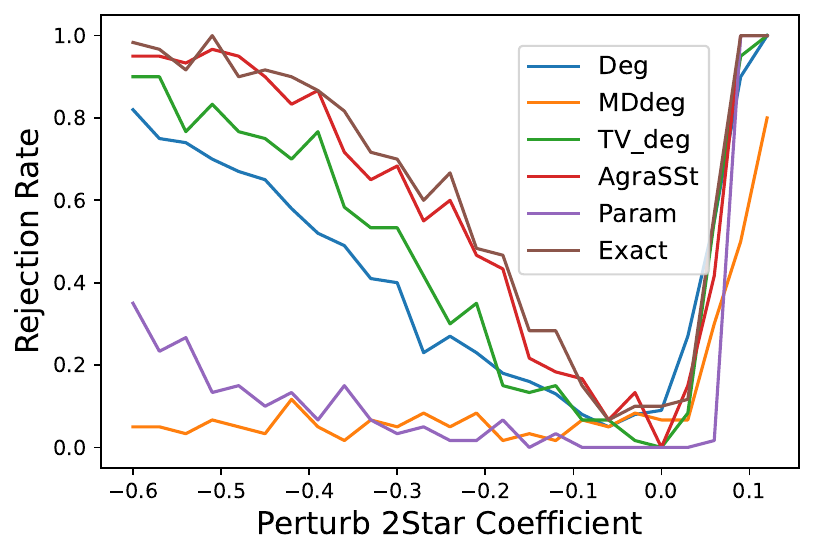}\label{fig:e2st_validation}}
    \subfigure[AgraSSt: different estimations]{
    \includegraphics[width=0.32\textwidth]{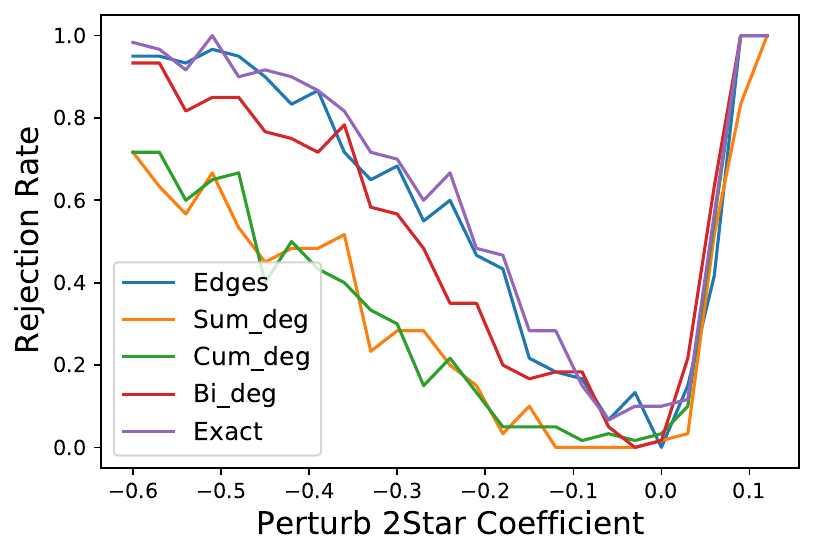}\label{fig:e2st_statistics}}
    \subfigure[AgraSSt with re-sampling]{
    \includegraphics[width=0.32\textwidth]{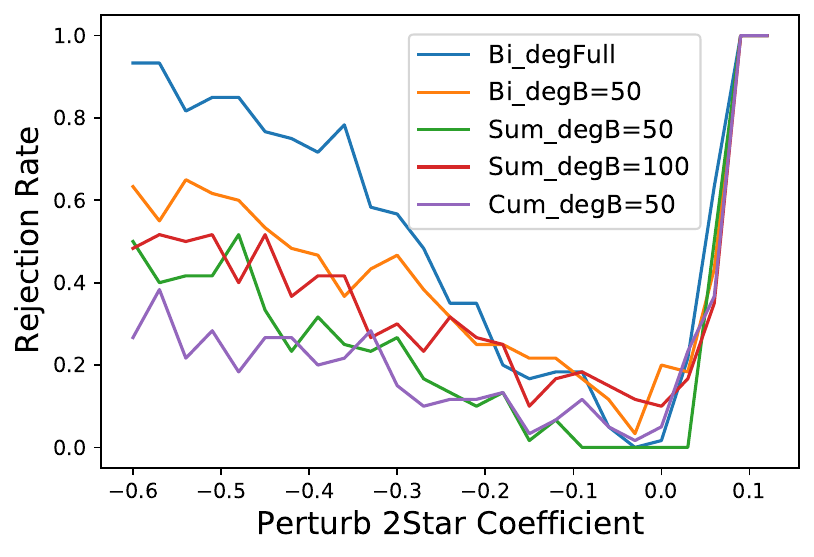}\label{fig:e2st_resampling}}
    \caption{Synthetic experiment on E2ST
    model in Eq.\eqref{eq:e2st_model}: $100$ trials; $\alpha=0.05$; $L=1000$.
    }
    \label{fig:e2st}
\end{figure*}


\subsection{Real-world applications on deep graph generators}\label{sec:realdata_exp}
We now assess the performances of a set of state-of-the-art deep generative models for graphs trained on ERGMs and 
{the} 
{Karate Club network collected by \citet{zachary1977information}.} The Karate Club network 
{has} 34 vertices 
{and} 78 edges 
representing friendships.
Soon after the data collection the Karate Club separated into two factions. 
This graph is a benchmark graph for community detection. {O}ne would not expect this graph to be 
close to an $G(n,p)$\footnote{Bernoulli random graph of size $n$, edge probability $p\in[0,1]$.}  graph 
{or {to be well} modelled by an ERGM. }

\subsubsection{Graph generation methods}
\textbf{GraphRNN} \citep{you2018graphrnn}
{is} an architecture to generate graphs from learning two recurrent neural network{s} (RNN), one {a}  vertex-level RNN and the other {an} edge-level RNN. The procedure starts from a breadth-first-search for vertex ordering; two RNNs are trained from a sequential procedure.
\textbf{NetGAN} \citep{bojchevski2018netgan} utilises an adversarial approach by training an interplay between {a} generator and {a} discriminator neural network on graph data.
\textbf{CELL} \citep{pmlr-v119-rendsburg20a} 
improves on {the} NetGAN idea by solving a low-rank {approximation} problem based on {a} cross-entropy objective.
\textbf{MC} is {the standard} Monte-Carlo  network 
sampling {in the}  \textit{ergm} suite in {\texttt{R}} and is used as a baseline when the simulated network is known to follow the 
model in Eq.\eqref{eq:ergm};
$q(x)$ needs to be known.

\subsubsection{Generator assessment results}

\begin{table}
    \centering
    \begin{tabular}{c|ccccc}
        \toprule
        {} & AgraSSt &  
        Deg & 
        MDdeg & 
        TV\_deg \\
        \midrule
         GraphRNN &  0.42 & 0.02 & 0.04 & 0.27 \\
         NetGAN & 0.81 & 0.13 & 0.61 & 0.54\\
         CELL & 0.05  & 0.06 & 0.09 & 0.12 \\
         \midrule
         MC & 0.04 & 0.03 & 0.02 &0.09\\
        \bottomrule
    \end{tabular}
    \caption{{R}ejection rates on various assessment approaches, with $L=1000$;
    200 samples to simulate the null; 100 trials; $\alpha = 0.05$. The higher the rejection rate, the worse {the model} fit.
    {MC is the baseline.}}
    \label{tab:deep_model_validation}
\end{table}

We first train the generative models with samples from ERGMs to assess their ability to generate ERGMs. 
The test results are shown in Table.\ref{tab:deep_model_validation}. 
From the result, we see that for the ``reliable'' \textbf{MC} 
generator all the assessment statistics presented have well-controlled type-I error.
Samples generated from \textbf{CELL} 
{deviate} 
not too far from the test level, 
{indicating} a good generative model for ERGMs. 
\textbf{NetGAN} and \textbf{GraphRNN} both encounter a high rejection rate, implying that the generated samples that are not close to the training E2ST {model}.

From the density based \textbf{AgraSSt}, {taking a $G(n,q)$ model,} we can interpret the model misfit by checking the estimated $\q$. For the true E2ST model to generate training samples, 
$q=0.112$, while CELL has $\q=0.116$ which is 
close to the null.
GraphRNN estimates $\q=0.128$ which  is substantially higher than the null.
{Although} GraphRNN is good in learning local patterns and structures for neighbour{hoods} \citep{you2018graphrnn}, 
it does not take the overall density {sufficiently} into account. Due to its limited ``look back'' and absence of ``look forward'' on the ordered vertex set during training, the over-generation of edges {may have} caused this significant difference for learning ERGMs. 
NetGAN, on the other hand, produces a close estimate  $\q=0.106$. However, {counting triangles,} 
it only has 
{on average} $12.6$ {triangles} , which is far less from the null with expected {number of triangles} $
46.3$. NetGAN, due to its random walk adversarial procedure, {may} not  {be}  effective in
learning {such}  clustered patterns. 

\subsubsection{Case study: Karate Club network}
\begin{wraptable}{r}{80mm}
\vspace{-.3cm}
    \centering
    \begin{tabular}{c|ccccc}
        \toprule
        {} & AgraSSt &  
        Deg & 
        MDdeg & 
        TV\_deg \\
        \midrule
         GraphRNN & {\color{red}0.00} & {\color{red}0.01} & 0.15 & {\color{red}0.00} \\
         NetGAN & {\color{red}0.00} & {\color{red}0.02} & 0.59 & {\color{red}0.00}  \\
         CELL & 0.34 & 0.09 & 0.17 & 0.61 \\
        \bottomrule
    \end{tabular}
    \caption{$p$-values for models trained on the Karate Club network; 100 samples to simulate the null distribution; rejection at 
    $\alpha=0.05$ is marked red. }
    \vspace{-0.2cm}
    \label{tab:karate_rej}
\end{wraptable}
Next, we assess the performances of these generative models by training on the 
Karate Club network \citep{zachary1977information}.
The $p$-values for different testing procedures are  shown in Table.\ref{tab:karate_rej}. 
From the results, we see that \AgraSSt{} rejects samples generated from both GraphRNN and NetGAN trained with the Karate Club network.  
Although the edge densities generated from the trained GraphRNN  
(edge density $15.3\%$) and NetGAN
(edge density $13.4\%$) are comparable with 
the Karate Club edge density 
of $13.9\%$, both GraphRNN and NetGAN samples exhibit a single large component rather than two fairly separated communities in the Karate Club network. This  difference is picked up by \AgraSSt{}, Deg and TV\_deg, which {all} reject both models. On the other hand, CELL generates samples that are not rejected by all tests at significance level $\alpha=0.05$.
{In Figure.\ref{fig:karate} in the SI, the Karate Club network is shown in Figure.\ref{fig:karate_club}.} Samples from GraphRNN, NetGAN and CELL  are shown in Figure.\ref{fig:karate_graphrnn}, \ref{fig:karate_netgan} and \ref{fig:karate_cell} respectively.

{SI \ref{app:exp} includes additional results and visualisations; a second case study --- {the} Florentine marriage network {from \citet{padgett1993robust}} --- is presented in SI \ref{app:karate}; 
additional visualisations {of} 
{the} reliable sample batch selection procedure described in \cref{sec:reliable-sample} 
{are} also included.}

\section{Discussions and future directions}\label{sec:future}

In this paper, we propose \AgraSSt, a unique general purpose model assessment and criticism procedure for implicit random graph models. {As it is} based on a kernel Stein statistic, {we are able to give theoretical guarantees.} 
{\AgraSSt{} {not only}  solves an important problem but also 
 opens up a whole set of follow-up}
research problems of which we list a few here.
(i).
Currently AgraSSt is only applied to undirected and unweighted graphs. Extensions to more general graphs as well as to time series of graphs will be interesting to explore in follow-up work.
(ii).
\AgraSSt{} {could} be also helpful to improve design and training of deep graph generative models, e.g. by regularising graph features {if}  there is a misfit.
(iii).
ERGMs allow for exogenous features to be included in the sufficient statistics.
{AgraSSt can be based on a variety of statistics $t(x)$; further examples are found in
\cref{app:additional}.}
{It {would also be possible} 
to incorporate exogenous features in the}
statistics $t(x)$ in AgraSSt, 
{for example using ideas from graph attention networks}
\citep{velivckovic2018graph}. Exploring this idea in more detail will be another topic of further research. 

{{As} AgraSSt depends on the chosen summary statistic $t(x)$,  results have to be interpreted with regards to the respective conditional distributions. Also multiple tests will have the $p$-values to be adjusted to avoid misinterpretation of tests, which could have serious consequences for example in the area of personal health.}

\bigskip
{
{\bf Acknowledgement.} The authors would like to thank Chris Oates for a helpful discussion which led to an improvement of the presentation, and to identifying a minor mistake in a previous version. G.R. and W.X. acknowledge the support from EPSRC grant EP/T018445/1.
G.R is also supported in part by EPSRC grants EP/W037211/1, EP/V056883/1, and EP/R018472/1.}
\bibliography{main}

\begin{thebibliography}{53}
\providecommand{\natexlab}[1]{#1}
\providecommand{\url}[1]{\texttt{#1}}
\expandafter\ifx\csname urlstyle\endcsname\relax
  \providecommand{\doi}[1]{doi: #1}\else
  \providecommand{\doi}{doi: \begingroup \urlstyle{rm}\Url}\fi

\bibitem[Bal{\'a}zs and T{\'o}th()]{balazs2014stirling}
M{\'a}rton Bal{\'a}zs and B{\'a}lint T{\'o}th.
\newblock Stirling’s formula and demoivre-laplace central limit theorem.
\newblock URL
  \url{https://people.maths.bris.ac.uk/~mb13434/Stirling_DeMoivre_Laplace.pdf}.

\bibitem[Berlinet and Thomas(2004)]{RKHSbook}
Alain Berlinet and Christine Thomas.
\newblock \emph{Reproducing {K}ernel {H}ilbert {S}paces in {P}robability and
  {S}tatistics}.
\newblock Kluwer Academic Publishers, 2004.

\bibitem[Besag(1975)]{besag1975statistical}
Julian Besag.
\newblock Statistical analysis of non-lattice data.
\newblock \emph{Journal of the Royal Statistical Society: Series D (The
  Statistician)}, 24\penalty0 (3):\penalty0 179--195, 1975.

\bibitem[Bhamidi et~al.(2011)Bhamidi, Bresler, and Sly]{bhamidi2011mixing}
Shankar Bhamidi, Guy Bresler, and Allan Sly.
\newblock Mixing time of exponential random graphs.
\newblock \emph{The Annals of Applied Probability}, 21\penalty0 (6):\penalty0
  2146--2170, 2011.

\bibitem[Bojchevski et~al.(2018)Bojchevski, Shchur, Z{\"u}gner, and
  G{\"u}nnemann]{bojchevski2018netgan}
Aleksandar Bojchevski, Oleksandr Shchur, Daniel Z{\"u}gner, and Stephan
  G{\"u}nnemann.
\newblock Net{GAN}: Generating graphs via random walks.
\newblock In \emph{International Conference on Machine Learning}, pages
  610--619. PMLR, 2018.

\bibitem[Bresler and Nagaraj(2018)]{bresler2018optimal}
Guy Bresler and Dheeraj Nagaraj.
\newblock Optimal single sample tests for structured versus unstructured
  network data.
\newblock In \emph{Conference On Learning Theory}, pages 1657--1690. PMLR,
  2018.

\bibitem[Bresler and Nagaraj(2019)]{bresler2019stein}
Guy Bresler and Dheeraj Nagaraj.
\newblock Stein’s method for stationary distributions of {M}arkov chains and
  application to {I}sing models.
\newblock \emph{The Annals of Applied Probability}, 29\penalty0 (5):\penalty0
  3230--3265, 2019.

\bibitem[Chatterjee and Diaconis(2013)]{chatterjee2013estimating}
Sourav Chatterjee and Persi Diaconis.
\newblock Estimating and understanding exponential random graph models.
\newblock \emph{The Annals of Statistics}, 41\penalty0 (5):\penalty0
  2428--2461, 2013.

\bibitem[Chwialkowski et~al.(2016)Chwialkowski, Strathmann, and
  Gretton]{chwialkowski2016kernel}
Kacper Chwialkowski, Heiko Strathmann, and Arthur Gretton.
\newblock A kernel test of goodness of fit.
\newblock In \emph{International Conference on Machine Learning}, pages
  2606--2615. PMLR, 2016.

\bibitem[Chwialkowski et~al.(2014)Chwialkowski, Sejdinovic, and
  Gretton]{chwialkowski2014wild}
Kacper~P Chwialkowski, Dino Sejdinovic, and Arthur Gretton.
\newblock A wild bootstrap for degenerate kernel tests.
\newblock In \emph{Advances in Neural Information Processing Systems}, pages
  3608--3616, 2014.

\bibitem[Dai et~al.(2020)Dai, Nazi, Li, Dai, and Schuurmans]{dai2020scalable}
Hanjun Dai, Azade Nazi, Yujia Li, Bo~Dai, and Dale Schuurmans.
\newblock Scalable deep generative modeling for sparse graphs.
\newblock In \emph{International Conference on Machine Learning}, pages
  2302--2312. PMLR, 2020.

\bibitem[Eldan and Gross(2018)]{eldan2018exponential}
Ronen Eldan and Renan Gross.
\newblock Exponential random graphs behave like mixtures of stochastic block
  models.
\newblock \emph{The Annals of Applied Probability}, 28\penalty0 (6):\penalty0
  3698--3735, 2018.

\bibitem[Finucan(1964)]{Finucan}
H.~M. Finucan.
\newblock The mode of a multinomial distribution.
\newblock \emph{Biometrika}, 51\penalty0 (3/4):\penalty0 513--517, 1964.

\bibitem[Frank and Strauss(1986)]{frank1986markov}
Ove Frank and David Strauss.
\newblock Markov graphs.
\newblock \emph{Journal of the American Statistical Association}, 81\penalty0
  (395):\penalty0 832--842, 1986.

\bibitem[Gorham and Mackey(2015)]{gorham2015measuring}
Jackson Gorham and Lester Mackey.
\newblock Measuring sample quality with {S}tein's method.
\newblock In \emph{Advances in Neural Information Processing Systems}, pages
  226--234, 2015.

\bibitem[Gorham et~al.(2020)Gorham, Raj, and Mackey]{gorham2020stochastic}
Jackson Gorham, Anant Raj, and Lester Mackey.
\newblock Stochastic {S}tein discrepancies.
\newblock \emph{Advances in Neural Information Processing Systems},
  33:\penalty0 17931--17942, 2020.

\bibitem[Gretton et~al.(2007)Gretton, Borgwardt, Rasch, Sch{\"o}lkopf, and
  Smola]{gretton2007kernel}
Arthur Gretton, Karsten Borgwardt, Malte Rasch, Bernhard Sch{\"o}lkopf, and
  Alex~J Smola.
\newblock A kernel method for the two-sample-problem.
\newblock In \emph{Advances in Neural Information Processing Systems}, pages
  513--520, 2007.

\bibitem[Gretton et~al.(2012)Gretton, Sejdinovic, Strathmann, Balakrishnan,
  Pontil, Fukumizu, and Sriperumbudur]{gretton2012optimal}
Arthur Gretton, Dino Sejdinovic, Heiko Strathmann, Sivaraman Balakrishnan,
  Massimiliano Pontil, Kenji Fukumizu, and Bharath~K Sriperumbudur.
\newblock Optimal kernel choice for large-scale two-sample tests.
\newblock In \emph{Advances in Neural Information Processing Systems}, pages
  1205--1213, 2012.

\bibitem[Hinton(2002)]{hinton2002training}
Geoffrey~E Hinton.
\newblock Training products of experts by minimizing contrastive divergence.
\newblock \emph{Neural Computation}, 14\penalty0 (8):\penalty0 1771--1800,
  2002.

\bibitem[Holland and Leinhardt(1981)]{holland1981exponential}
Paul~W Holland and Samuel Leinhardt.
\newblock An exponential family of probability distributions for directed
  graphs.
\newblock \emph{Journal of the American Statistical Association}, 76\penalty0
  (373):\penalty0 33--50, 1981.

\bibitem[Hunter and Handcock(2006)]{hunter2006inference}
David~R Hunter and Mark~S Handcock.
\newblock Inference in curved exponential family models for networks.
\newblock \emph{Journal of Computational and Graphical Statistics}, 15\penalty0
  (3):\penalty0 565--583, 2006.

\bibitem[Hunter et~al.(2008)Hunter, Goodreau, and Handcock]{hunter2008goodness}
David~R Hunter, Steven~M Goodreau, and Mark~S Handcock.
\newblock Goodness of fit of social network models.
\newblock \emph{Journal of the American Statistical Association}, 103\penalty0
  (481):\penalty0 248--258, 2008.

\bibitem[Hyv{\"a}rinen(2006)]{hyvarinen2006consistency}
Aapo Hyv{\"a}rinen.
\newblock Consistency of pseudolikelihood estimation of fully visible
  {B}oltzmann machines.
\newblock \emph{Neural Computation}, 18\penalty0 (10):\penalty0 2283--2292,
  2006.

\bibitem[Jiang et~al.(2018)Jiang, Wu, Jin, and Wong]{jiang2018convergence}
Bai Jiang, Tung-Yu Wu, Yifan Jin, and Wing~H Wong.
\newblock Convergence of contrastive divergence algorithm in exponential
  family.
\newblock \emph{The Annals of Statistics}, 46\penalty0 (6A):\penalty0
  3067--3098, 2018.

\bibitem[Jitkrittum et~al.(2016)Jitkrittum, Szab{\'o}, Chwialkowski, and
  Gretton]{jitkrittum2016interpretable}
Wittawat Jitkrittum, Zolt{\'a}n Szab{\'o}, Kacper~P Chwialkowski, and Arthur
  Gretton.
\newblock Interpretable distribution features with maximum testing power.
\newblock In \emph{Advances in Neural Information Processing Systems}, pages
  181--189, 2016.

\bibitem[Jitkrittum et~al.(2017)Jitkrittum, Xu, Szab{\'o}, Fukumizu, and
  Gretton]{jitkrittum2017linear}
Wittawat Jitkrittum, Wenkai Xu, Zolt{\'a}n Szab{\'o}, Kenji Fukumizu, and
  Arthur Gretton.
\newblock A linear-time kernel goodness-of-fit test.
\newblock In \emph{Advances in Neural Information Processing Systems}, pages
  262--271, 2017.

\bibitem[Ley et~al.(2017)Ley, Reinert, and Swan]{ley2017stein}
Christophe Ley, Gesine Reinert, and Yvik Swan.
\newblock Stein’s method for comparison of univariate distributions.
\newblock \emph{Probability Surveys}, 14:\penalty0 1--52, 2017.

\bibitem[Li et~al.(2018)Li, Vinyals, Dyer, Pascanu, and
  Battaglia]{li2018learning}
Yujia Li, Oriol Vinyals, Chris Dyer, Razvan Pascanu, and Peter Battaglia.
\newblock Learning deep generative models of graphs.
\newblock \emph{arXiv preprint arXiv:1803.03324}, 2018.

\bibitem[Liao et~al.(2019)Liao, Li, Song, Wang, Hamilton, Duvenaud, Urtasun,
  and Zemel]{liao2019efficient}
Renjie Liao, Yujia Li, Yang Song, Shenlong Wang, Will Hamilton, David~K
  Duvenaud, Raquel Urtasun, and Richard Zemel.
\newblock Efficient graph generation with graph recurrent attention networks.
\newblock \emph{Advances in Neural Information Processing Systems},
  32:\penalty0 4255--4265, 2019.

\bibitem[Liu et~al.(2020)Liu, Xu, Lu, Zhang, Gretton, and
  Sutherland]{liu2020learning}
Feng Liu, Wenkai Xu, Jie Lu, Guangquan Zhang, Arthur Gretton, and Danica~J
  Sutherland.
\newblock Learning deep kernels for non-parametric two-sample tests.
\newblock In \emph{International Conference on Machine Learning}, pages
  6316--6326. PMLR, 2020.

\bibitem[Liu et~al.(2021)Liu, Xu, Lu, and Sutherland]{liu2021meta}
Feng Liu, Wenkai Xu, Jie Lu, and Danica~J Sutherland.
\newblock Meta two-sample testing: Learning kernels for testing with limited
  data.
\newblock \emph{Advances in Neural Information Processing Systems}, 34, 2021.

\bibitem[Liu et~al.(2016)Liu, Lee, and Jordan]{liu2016kernelized}
Qiang Liu, Jason Lee, and Michael Jordan.
\newblock A kernelized {S}tein discrepancy for goodness-of-fit tests.
\newblock In \emph{International Conference on Machine Learning}, pages
  276--284, 2016.

\bibitem[Lospinoso and Snijders(2019)]{lospinoso2019goodness}
Josh Lospinoso and Tom~AB Snijders.
\newblock Goodness of fit for stochastic actor-oriented models.
\newblock \emph{Methodological Innovations}, 12\penalty0 (3):\penalty0
  2059799119884282, 2019.

\bibitem[Mukherjee and Xu(2013)]{mukherjee2013statistics}
Sumit Mukherjee and Yuanzhe Xu.
\newblock Statistics of the two-star {ERGM}.
\newblock \emph{arXiv preprint arXiv:1310.4526}, 2013.

\bibitem[Ouadah et~al.(2020)Ouadah, Robin, and Latouche]{ouadah2020degree}
Sarah Ouadah, St{\'e}phane Robin, and Pierre Latouche.
\newblock Degree-based goodness-of-fit tests for heterogeneous random graph
  models: Independent and exchangeable cases.
\newblock \emph{Scandinavian Journal of Statistics}, 47\penalty0 (1):\penalty0
  156--181, 2020.

\bibitem[Padgett and Ansell(1993)]{padgett1993robust}
John~F Padgett and Christopher~K Ansell.
\newblock Robust {A}ction and the {R}ise of the {M}edici, 1400-1434.
\newblock \emph{American Journal of Sociology}, 98\penalty0 (6):\penalty0
  1259--1319, 1993.

\bibitem[Reinert and Ross(2019)]{reinert2019approximating}
Gesine Reinert and Nathan Ross.
\newblock Approximating stationary distributions of fast mixing {G}lauber
  dynamics, with applications to exponential random graphs.
\newblock \emph{The Annals of Applied Probability}, 29\penalty0 (5):\penalty0
  3201--3229, 2019.

\bibitem[Rendsburg et~al.(2020)Rendsburg, Heidrich, and
  Luxburg]{pmlr-v119-rendsburg20a}
Luca Rendsburg, Holger Heidrich, and Ulrike~Von Luxburg.
\newblock {N}et{GAN} without {GAN}: From random walks to low-rank
  approximations.
\newblock In \emph{Proceedings of the 37th International Conference on Machine
  Learning}, pages 8073--8082. PMLR, 2020.

\bibitem[Schmid and Desmarais(2017)]{schmid2017exponential}
Christian~S Schmid and Bruce~A Desmarais.
\newblock Exponential random graph models with big networks: Maximum
  pseudolikelihood estimation and the parametric bootstrap.
\newblock In \emph{2017 IEEE International Conference on Big Data}, pages
  116--121. IEEE, 2017.

\bibitem[Shalizi and Rinaldo(2013)]{shalizi2013consistency}
Cosma~Rohilla Shalizi and Alessandro Rinaldo.
\newblock Consistency under sampling of exponential random graph models.
\newblock \emph{Annals of Statistics}, 41\penalty0 (2):\penalty0 508, 2013.

\bibitem[Shervashidze et~al.(2011)Shervashidze, Schweitzer, Leeuwen, Mehlhorn,
  and Borgwardt]{shervashidze2011weisfeiler}
Nino Shervashidze, Pascal Schweitzer, Erik Jan~van Leeuwen, Kurt Mehlhorn, and
  Karsten~M Borgwardt.
\newblock Weisfeiler-{L}ehman graph kernels.
\newblock \emph{Journal of Machine Learning Research}, 12\penalty0
  (Sep):\penalty0 2539--2561, 2011.

\bibitem[Snijders(2002)]{snijders2002markov}
Tom~AB Snijders.
\newblock Markov chain {M}onte {C}arlo estimation of exponential random graph
  models.
\newblock \emph{Journal of Social Structure}, 3\penalty0 (2):\penalty0 1--40,
  2002.

\bibitem[Strauss and Ikeda(1990)]{strauss1990pseudolikelihood}
David Strauss and Michael Ikeda.
\newblock Pseudolikelihood estimation for social networks.
\newblock \emph{Journal of the American Statistical Association}, 85\penalty0
  (409):\penalty0 204--212, 1990.

\bibitem[Tem{\v{c}}inas et~al.(2021)Tem{\v{c}}inas, Nanda, and
  Reinert]{temvcinas2021multivariate}
Tadas Tem{\v{c}}inas, Vidit Nanda, and Ges~ine Reinert.
\newblock Multivariate central limit theorems for random clique complexes.
\newblock \emph{arXiv preprint arXiv:2112.08922}, 2021.

\bibitem[Van~Duijn et~al.(2009)Van~Duijn, Gile, and Handcock]{van2009framework}
Marijtje~AJ Van~Duijn, Krista~J Gile, and Mark~S Handcock.
\newblock A framework for the comparison of maximum pseudo-likelihood and
  maximum likelihood estimation of exponential family random graph models.
\newblock \emph{Social Networks}, 31\penalty0 (1):\penalty0 52--62, 2009.

\bibitem[Veli{\v{c}}kovi{\'c} et~al.(2018)Veli{\v{c}}kovi{\'c}, Cucurull,
  Casanova, Romero, Li{\`o}, and Bengio]{velivckovic2018graph}
Petar Veli{\v{c}}kovi{\'c}, Guillem Cucurull, Arantxa Casanova, Adriana Romero,
  Pietro Li{\`o}, and Yoshua Bengio.
\newblock Graph attention networks.
\newblock In \emph{International Conference on Learning Representations}, 2018.

\bibitem[Wasserman and Faust(1994)]{wasserman1994social}
Stanley Wasserman and Katherine Faust.
\newblock \emph{Social {N}etwork {A}nalysis: {M}ethods and {A}pplications},
  volume~8.
\newblock Cambridge University Press, 1994.

\bibitem[Xu and Matsuda(2020)]{xu2020stein}
Wenkai Xu and Takeru Matsuda.
\newblock A {S}tein goodness-of-fit test for directional distributions.
\newblock \emph{International Conference on Artificial Intelligence and
  Statistics}, 2020.

\bibitem[Xu and Matsuda(2021)]{xu2021interpretable}
Wenkai Xu and Takeru Matsuda.
\newblock Interpretable {S}tein goodness-of-fit tests on {R}iemannian
  manifolds.
\newblock \emph{International Conference on Machine Learning}, 2021.

\bibitem[Xu and Reinert(2021)]{xu2021stein}
Wenkai Xu and Gesine Reinert.
\newblock A {S}tein goodness-of-test for exponential random graph models.
\newblock In \emph{International Conference on Artificial Intelligence and
  Statistics}, pages 415--423. PMLR, 2021.

\bibitem[Yang et~al.(2018)Yang, Liu, Rao, and Neville]{yang2018goodness}
Jiasen Yang, Qiang Liu, Vinayak Rao, and Jennifer Neville.
\newblock Goodness-of-fit testing for discrete distributions via {S}tein
  discrepancy.
\newblock In \emph{International Conference on Machine Learning}, pages
  5557--5566, 2018.

\bibitem[You et~al.(2018)You, Ying, Ren, Hamilton, and
  Leskovec]{you2018graphrnn}
Jiaxuan You, Rex Ying, Xiang Ren, William Hamilton, and Jure Leskovec.
\newblock Graph{RNN}: {G}enerating realistic graphs with deep auto-regressive
  models.
\newblock In \emph{International conference on machine learning}, pages
  5708--5717. PMLR, 2018.

\bibitem[Zachary(1977)]{zachary1977information}
Wayne~W Zachary.
\newblock An information flow model for conflict and fission in small groups.
\newblock \emph{Journal of Anthropological Research}, 33\penalty0 (4):\penalty0
  452--473, 1977.

\end{thebibliography}
\bibliographystyle{plainnat}

\clearpage
\appendix

\section{Proofs and additional theoretical results}\label{app:proof}

\subsection{Proofs
} \label{siproofs} 

\paragraph{Proof of \cref{Stein operators}}

{For convenience we repeat the statement of the lemma here.}

\medskip
{
{\bf Lemma \ref{Stein operators}} 
{\it 
In this setting, $\mathcal{A}_{q,{\Delta_s t(x)}={\uk}}^{(s)}$ is a Stein operator for the conditional distribution of $X$ given ${\Delta_s t}(X)= \uk$, and 
$ \sum_s \mathcal{A}_{q,{\Delta_s t(x)} ={\uk}}^{(s)}$ is a Stein operator for 
the conditional distribution of $X$ given ${\Delta_s t(X)}=\uk$. 
} } 

\medskip 
\begin{proof}  
{In order to show the assertion we  prove that for $\mathbb{E}_{{q,}{\Delta_s t(X)}=\uk}$ denoting the conditional distribution of $X$ given ${\Delta_s t}(X) = \uk$,  the expectation $\mathbb{E}_{{q,}{\Delta_s t}(X)=\uk} [\mathcal{A}_{q,t}^{(s)} f ]$ vanishes for all functions for which the  expectation exists; {again we use the abbreviation 
$\mathcal{A}_{q,t}= \mathcal{A}_{q,\Delta_s t} $}. Let $f$ be such a function.} We have 
\begin{align*}\mathcal{A}_{q,t}^{(s)} f(x^{s,1}) 
 = & q(x^{s,0} | {\Delta_s t(x)} {=\uk} ) (f( x^{(s,0)})  - f(x^{s,1}))\\
 \mathcal{A}_{q,t}^{(s)} f(x^{s,0}) 
 = & q(x^{s,1} | {\Delta_s t(x)} {=\uk}) (f( x^{(s,1)})  - f(x^{s,0})). 
\end{align*}
Thus, 
$\mathbb{E}_{q,t}$,
\begin{align*}
\mathbb{E}_{{q,}t} [\mathcal{A}_{q,t}^{(s)} f  {]} 
 = & \sum_{x_{-s}} \left\{ \mathbb{1} (x_s=1)
 p_{t} ( x^{(s,1)}) q(y^{s,0} | {\Delta_s t(y)}{=\uk} )  (f( x^{(s,0)})  - f(x^{s,1}))  \right.\\
 & \left. -  \mathbb{1} (x_s=0)
 p_{t} ( x^{(s,0)}) q(y^{s,1} | {\Delta_s t(y)}{=\uk}) (f( x^{(s,0)})  - f(x^{s,1})) \right\}\\
 =& \sum_{x_{-s}}(f( x^{(s,0)})  - f(x^{s,1}))  p_{t} ( x^{(s,1)}) p_{t} ( x^{(s,0)}) \left\{ 
 \mathbb{1} (x_s=1)   -  \mathbb{1} (x_s=0)   
 \right\} \\ 
 =& 0
 .
\end{align*}
\end{proof}  

\paragraph{Proof of Theorem \ref{th:consist}} 

For convenience we repeat the theorem here. 

{\bf Theorem \ref{th:consist}} 
{\it 
Assume that $\q_t (x^{(s,1)})  $ is a consistent estimator for $q_t(x^{(s,1)}) $ as $L \rightarrow \infty$. Then for any function $f$ such that  $|| \Delta f|| < \infty$ we have $\E_q[\A_{ \q, t}f(x)] \to \E_q[\A_{q,t}f(x) ] =0$ as $L \rightarrow \infty$.}

\begin{proof}
{We recall the notation that
\cref{eq:cond_prob}.
} 
We have that 
$$\A_{q,t} f (x)  = \frac1N \sum_{s \in [N]} 
[ q (x^{(s,1)}|{\Delta_s t(x)} f(x^{(s,1)}) + q (x^{(s,0)}| {\Delta_s t(x)}) f ( x^{s,0}) - f(x) ]$$
so that 
\begin{align*}
   \A_{{\widehat q}, {t} } f(x)  - \A_{q,{t} } f (x) 
   =  \frac1N \sum_{s \in [N]} & \{ (  {\widehat q}(x^{(s,1)}| {\Delta_s t(x)} )- q(x^{(s,1)}| {\Delta_s t(x)})
   f(x^{(s,1)})  \\
   & + [( 1 - {\widehat q}(x^{(s,1)}| {\Delta_s t(x)}) - (1-q(x^{(s,1)}| {\Delta_s t(x)}) )] f ( x^{s,0})  \}
   \\
   \end{align*}
    \begin{align*}&= \frac1N \sum_{s \in [N]}   ( {\widehat q}(x^{(s,1)}| {\Delta_s t(x)}) - q(x^{(s,1)}| {\Delta_s t(x)}))  \Delta_s f(x). 
\end{align*}
Hence 
$$
|  \A_{{\widehat q}, {t} } f(x)  - \A_{q,{t} }  f (x)| \le || \Delta f|| 
\frac1N \sum_{s \in [N]}   | {\widehat q}(x^{(s,1)}|{\Delta_s t(x)} ) - q(x^{(s,1)}| {\Delta_s t(x)} ) |.
$$
Thus, if for all $s \in [N]$, as $L\rightarrow \infty$ we have $  {\widehat q}(x^{(s,1)}| {\Delta_s t(x)} ) - q(x^{(s,1)}|{\Delta_s t(x)})  \rightarrow 0$ for all $s$ then so does $|  \A_{{\widehat q}, {t} } f(x)  - \A_{q,{t} }f  (x)|$. The assertion follows {from the assumption that $\q_t(x^{s,1})$
is a consistent estimator for 
$q_t(x^{s,1})$ 
as $L \rightarrow \infty$}. 
\end{proof}

\paragraph{Proof of Theorem  \ref{consistenterop}}

{For convenience we repeat the statement of the theorem here. Recall that  a random graph model is {\it edge-exchangeable} if its edge indicator variables are finitely exchangeable. Often we just write {\it edge-exchangeable graph}. An ERGM is an example of an edge-exchangable graph. }

\medskip
{
{\bf Theorem \ref{consistenterop}} 
{\it {If the graph is edge-exchangeable, then}
$ {\AgraSSt}^2 (\q,t;x) $
is a consistent estimator of 
$$
{\rm{gKSS}}^2 (q; x) = N^{-2}  \sum_{s, s'\in [N]} 
\left\langle \A^{(s)}_{ q}  K(x,\cdot), \A^{(s')}_{ q}  K(\cdot,x)\right\rangle_{\H}.$$
}}

{
For easier tractability 
the proof is organised in {two steps}}.

\begin{enumerate}
    \item 
First, 
\cref{siconsistency2} shows that in an edge-exchangeable random graph model, 
$g_{\uk}$ given in \cref{gk} is a consistent estimator for $q (x^{(s,1)}| {\Delta_s t(x)} ) $  as $L \rightarrow \infty.$
\item {\cref{consistentopexpand} uses these results to obtain a concentration bound for $\A_{\widehat{q}} $ from which then Theorem  \ref{consistenterop}} follows. 
\end{enumerate}
Moreover  theoretical guarantees for fixed $L$ which depend on the model  are given. As the graph generator can generate as large a number $L$ {of graphs} as desired, these theoretical results can be used to determine $L$ which result in theoretical guarantees on deviations from the mean.

\begin{proposition}
\label{siconsistency2} 
Suppose that $X_1, \ldots, X_L$ are i.i.d.\,copies of the adjacency matrix of an edge-exchangeable  random graph model. Let 
$s=(i,j)$ be a fixed vertex-pair. For $l=1, \ldots, L$ and for a graph $X_l$ let ${\Delta_s t(x)}_l$ denote {the version of the}  {possibly multivariate} statistic $\Delta_s t$ which is evaluated on the collection of indicator variables in $X_l$ except $X_{s, l}$. {For a possible ${\Delta_s t(x)}_l$ outcome $\uk$}, let 
$p(\uk) = \mathbb{P} ({\Delta_s t(x)}_l=\uk) $ and let $\uk$ be such that $p(\uk) \ne 0$.
 Set 
$$ p ( 1; \uk) = \mathbb{P}( | X_s=1|{\Delta_s t(x)} = \uk);$$ 
let 
$$n(\uk,s)=\sum_{l=1}^L X_{s,l} {\mathbb{1}}({\Delta_s t(x)}_l= \uk) \quad \text{ and } \quad n(\uk^{(s)})=\sum_{l=1}^L  {\mathbb{1}}({\Delta_s t(x)}_l = \uk );  $$
$$n(\uk) = \sum_{s \in [N]} n(\uk,s) \quad \text{ and } \quad
N_{\uk} = \sum_{s \in [N]} n(\uk^{(s)});$$
and set
$${g}(\uk) = \frac{n(\uk)}{N_{\uk}}
{\mathbb{1}}( N_{\uk} \ge 1).$$
{We abbreviate
$$ \sigma_n^2(\uk) =  Var \left(\sum_{s \in [N]} X_{s,l} {\mathbb{1}}(t^{(s)}_l= \uk) \right) , \quad 
\sigma_N ^2(\uk) = Var \left(\sum_{s \in [N]}  {\mathbb{1}}(t^{(s)}_l= \uk) \right).  $$} 
Then 
${g}(\uk) \rightarrow p$
in probability as $ L \rightarrow \infty $. In particular, for all $\epsilon >0$, 
$$\mathbb{P} \left[ \left| \hat{g}(\uk) - {  p ( 1; \uk)} \right| > \epsilon \right] 
\le 
\frac{4}{\epsilon^2 N^2  L }   \{  \sigma_n^2(\uk) + \sigma_N^2(\uk)  \} 
.$$ 
\end{proposition} 


\begin{proof} 
{Due to the exchangeability of the edges} we have, with $s$ denoting a generic edge, 
{$$\mathbb{E} ( n(\uk)) = N \mathbb{E} (n (\uk, s)) = N   L p(1; \uk) p(\uk); \quad {\mathbb{E} ( N_{\uk}) = N \mathbb{E} (n (\uk^{(s)})) = N  L p(\uk)}.$$
{Moreover, due to the independence of $X_1, \ldots, X_L$, 
$$ Var (n (\uk)) = L  \sigma_n^2(\uk), 
\quad \quad Var (N_{\uk}) =   L \sigma_N ^2(\uk). $$}
{To show convergence in probability, let $\epsilon >0$. Then  
\begin{align*}
\mathbb{P} \left[ \left| \hat{g}(\uk) - p(1, \uk) \right| > \epsilon \right] 
\le & \mathbb{P} \left[ \left| \frac{n(\uk)  }{\mathbb{E}( n(\uk))} {\mathbb{1}} [ n(\uk) \ge 1 ]  - p(1, \uk) \right| >  \frac12\epsilon \right] \\
& + \mathbb{P} \left[ \left| n(\uk)  {\mathbb{1}} [ n(\uk) \ge 1 ] \left(\frac{1}{n(\uk)} - \frac{1}{\mathbb{E}( n(\uk))} \right)  \right| >  \frac12 \epsilon \right]. 
\end{align*}
Note that
$ n(\uk)  {\mathbb{1}} [ n(\uk) \ge 1 ] = n(\uk) .$
} 
By Chebychev's inequality, { 
\begin{align*}
\mathbb{P} \left[ \left| \frac{n(\uk)}{\mathbb{E}( N_{\uk}) }  {\mathbb{1}} [ N_{\uk} \ge 1 ] - {p(1; \uk)}  \right| >  \frac12\epsilon \right] 
&\le  \frac{4}{\epsilon^2 N^2 L^2 p(\uk)^2}   L \sigma_n^2(\uk) \\
&= \frac{4}{\epsilon^2 N^2 L} \sigma_n^2(\uk) 
\end{align*}
}
and {similarly
\begin{align*}
 \mathbb{P} \left[ \left| n(\uk)  {\mathbb{1}} [ N_{\uk} \ge 1 ] \left(\frac{1}{n(\uk)} - \frac{1}{\mathbb{E}( N_{\uk})} \right)  \right| >  \frac12 \epsilon \right]
& \le 
\mathbb{P} \left[ \frac{1}{ \mathbb{E}( N_{\uk})}  \left| \mathbb{E}(N_{\uk})  - N_{\uk}  \right| >  \frac12 \epsilon \right] \\
& \le  \frac{4}{\epsilon^2 N^2 L} \sigma_N^2(\uk) .
\end{align*} }}
Summing the contributions completes the proof. 
\end{proof}

\bigskip 
{Proposition \ref{siconsistency2} proves consistency as $L$, the number of samples, tends to infinity. 
If the edge-exchangeable model is an $ER(n,p)$ graph then consistency can be shown even within a single network sample, as follows.}
{
\begin{proposition}
\label{siconsistencyER} 
Suppose that $X$ is the adjacency matrix of an $ER(n,\astar)$ graph. Let 
$s=(i,j)$ be a fixed vertex-pair. For a possible ${\Delta_s t(x)}$ outcome $\uk$, let 
$p(\uk) = \mathbb{P} ({\Delta_s t(x)}=\uk) $ and let $\uk$ be such that $p(\uk) \ne 0$.
Let 
\begin{equation*}
    \tilde{p}=   \tilde{p}(\uk, X) := \frac{\sum_{s \in [N]} X_s \bone (\Delta_s t (X) = \uk) }{\sum_{u \in [N]}  \bone (\Delta_u t (X) = \uk) }
    = \astar + \sum_{s \in [N]} \frac{(X_s - \astar) \bone (\Delta_s t (X) = \uk)}{1 + \sum_{u \ne s } \bone (\Delta_u t (X) = \uk) }. 
\end{equation*}
Then 
${\tilde{p}}(\uk,X) \rightarrow \astar$
in probability as $ n \rightarrow \infty $.
\end{proposition} 
}

\begin{proof}
{The task is to show that 
$$ R(\uk, X) := \sum_{s \in [N]} \frac{(X_s - \astar) \bone (\Delta_s t (X) = \uk)}{1 + \sum_{u \ne s } \bone (\Delta_u t (X) = \uk) } \rightarrow 0 $$
in probability as $n \rightarrow \infty$.
    Following Remark 1.9 in \cite{reinert2019approximating},
let~$H$ and~$x$ be graphs with vertex and edge sets $V(H), V(x), E(H), E(x)$, and let
${I}(H,x)$  be the set of all injections $i_{H, x}: V(H) \rightarrow V(x)$. For such an injection and any edge ${e}=(u,v)$ {of~$H$}, we use the notation $i_{H,x}({e}) = {\{}i_{H,x}(u), i_{H,x}(v){\}}$ and $E(i_{H,x}) = \cup_{{e} \in E(H)} \{i_{H,x}({e})\}$. Then the number of edge-preserving injections of ${I}(H,x)$ is 
$$
t(H,x) =\sum_{i_{H,x} \in {I}(H,x)} \prod_{e \in E( H)} {\mathbbm 1}\bclr{i_{H,x}(e) \in E(x)}.
$$
Let $x/s := x\setminus \{s\}$ denote the collection of edge indicators except $x_s$. This is a collection using the same vertex set as $s$, 
and for $s \in E({x})$, 
\begin{eqnarray} 
\Delta_s t(H,x)
&=&\sum_{i_{H,x} \in {I}(H,x)} \left\{ \prod_{e \in  E( H)}  {\mathbbm 1}  \bclr{i_{H,x}(e) \in E(x/s^{s,1} )}
- \prod_{e \in H}  {\mathbbm 1}  \bclr{i_{H,x}(e) \in E(x/s^{s,0} )} \right\} \nonumber \\
&= & \sum_{i_{H,x} \in I(H,x)}   {\mathbbm 1} \bclr{s \in E(i_{H,x}) } \prod_{e \in  E( H) \setminus {{i_{H,x}^{-1}(s)}} } {\mathbbm 1}  \bbclr{i_{H,x}(e) \in \bclr{ E(x) \setminus i_{H,x}(s)}}.\label{eq:rep}
\end{eqnarray} 
We  define $$ 
\Delta_u t(H,x/s)
= \sum_{i_{H,x} \in I(H,x)}   {\mathbbm 1} \bclr{u \in E(i_{H,x}) } \prod_{e \in  E( H) \setminus {{i_{H,x}^{-1}(s)}} } {\mathbbm 1}  \bbclr{i_{H,x}(e) \in \bclr{ E(x) \setminus \{ s, i_{H,x}(u) \}}
 }.$$ 
Then $\mathbb{1}(Z_s = 1) $ is independent of $\Delta_u t(H,Z/s)$. 
Moreover, 
\begin{eqnarray*}
 \lefteqn{ R(\uk, X) }\\
 &=&   \sum_{s \in [N]}  (Z_s - \astar) \left\{ \mathbb{1} (\Delta_s t(Z) = \uk) \frac{1}{1 + \sum_{u: u \ne s} \mathbb{1} (\Delta_u t(Z / s) = \uk)} \right\} \\
  &&+ \sum_{s \in [N]}     (Z_s - \astar) \mathbb{1} (\Delta_s t(Z) = \uk)
\left( \frac{1}{1 + \sum_{u: u \ne s} \mathbb{1} (\Delta_u t(Z) = \uk)} - \frac{1}{1 + \sum_{u: u \ne s} \mathbb{1} (\Delta_u t(Z / s) = \uk)}\right) 
 \\
  &=&  \sum_{s \in [N]}  (Z_s - \astar)   \mathbb{1} (\Delta_s t(Z) = \uk)  \left( \frac{1}{1 + \sum_{u: u \ne s} \mathbb{1} (\Delta_u t(Z) = \uk)} - \frac{1}{1 + \sum_{u: u \ne s} \mathbb{1} (\Delta_u t(Z / s) = \uk)}\right) 
   \\
  &&+\sum_{s \in [N]}     (Z_s - \astar)  \mathbb{1} (\Delta_s t(Z) = \uk)
\left( \frac{1}{1 + \sum_{u: u \ne s} \mathbb{1} (\Delta_u t(Z) = \uk)} - \frac{1}{1 + \sum_{u: u \ne s} \mathbb{1} (\Delta_u t(Z / s) = \uk)}\right)
 \\
  &=&  \sum_{s \in [N]} (Z_s - \astar)   \mathbb{1} (\Delta_s t(Z) = \uk)
  \left( \frac{\sum_{u: u \ne s} \mathbb{1} (\Delta_u t(Z / s) = \uk) - \sum_{u: u \ne s} \mathbb{1} (\Delta_u t(Z ) = \uk)}{(1 + \sum_{w: w\ne s} \mathbb{1} (\Delta_w t(Z) = \uk))(1 + \sum_{w: w \ne s} \mathbb{1} (\Delta_w t(Z / s) = \uk)}\right) 
 .
\end{eqnarray*}
}
{
Using the product formula
$$\prod_{j=0}^{n} a_j - \prod_{j=0}^{n} b_j 
 = \sum_{k=0}^n \left\{ \prod_{j=0}^{k-1} a_j \right\} (a_k - b_k) \left\{ \prod_{j=k+1}^{n} b_j \right\},$$ with $\uk = (k_1, \ldots, k_\ell),$
\begin{eqnarray*}
    \lefteqn{\sum_{u: u \ne s} \mathbb{1} (\Delta_u t(Z / s) = \uk) - \sum_{u: u \ne s} \mathbb{1} (\Delta_u t(Z ) = \uk)}\\
&=& \sum_{u: u \ne s} \left\{
\prod_{\ell=1}^L \mathbb{1} (\Delta_u t_\ell (Z / s) = k_\ell) - \prod_{\ell=1}^L \mathbb{1} (\Delta_u t_\ell (Z) = k_\ell)
\right\} \\
&=& \sum_{u: u \ne s}
\left\{\sum_{d=1}^L 
\prod_{\ell=1}^{d-1} \mathbb{1} (\Delta_u t_\ell (Z / s) = k_\ell) ( \mathbb{1} (\Delta_u t_d (Z / s) = k_d)  - \mathbb{1} (\Delta_u t_d (Z) = k_d)
) \right.\\
&&\left. 
\prod_{\ell=d+1}^L \mathbb{1} (\Delta_u t_\ell (Z) = k_\ell)
\right\}.
\end{eqnarray*}
}

 By construction, $H_1$ is the single edge, and $k_1 =1$ is the only possible value. Hence
\begin{eqnarray*}
    \lefteqn{\sum_{u: u \ne s} \mathbb{1} (\Delta_u t(Z / s) = \uk) - \sum_{u: u \ne s} \mathbb{1} (\Delta_u t(Z ) = \uk)}\\
&=& \sum_{u: u \ne s}
\left\{\sum_{d=2}^L 
\prod_{\ell=1}^{d-1} \mathbb{1} (\Delta_u t_\ell (Z / s) = k_\ell) ( \mathbb{1} (\Delta_u t_d (Z / s) = k_d)  - \mathbb{1} (\Delta_u t_d (Z) = k_d)
)  \right.\\
&&\left. 
\prod_{\ell=d+1}^L \mathbb{1} (\Delta_u t_\ell (Z) = k_\ell)
\right\}.
\end{eqnarray*}

For convergence, it thus suffices to consider $\mathbb{1} (\Delta_u t_\ell (Z / s) = k_\ell) ( \mathbb{1} (\Delta_u t_d (Z / s) = k_d)  - \mathbb{1} (\Delta_u t_d (Z) = k_d) $ for single graphs $H_d$ with at least 3 vertices.

{
We first treat the case that there is an $\ell \ge 2$ such that $k_\ell=0$. Then $H_\ell$ has at least 3 vertices and at least 2 edges. 
As there are of order $n^{v(H_ell) - 2} \ge n $ 
possible realisations of $H_\ell$ in $Z$ which contain the edge $s$, we have 
$$ \mathbb{P} (\Delta_s t_\ell (Z) = 0 )
\le (1-\astar)^{n e(H_\ell) } .$$
This quantity tends to 0 exponentially fast. A similar argument holds for $k_\ell =1$. Thus, the values of $k_\ell$ which we consider are such that the unscaled counts grow with $n$.}

{The next step is a  normal approximation for $\Delta_s t (Z),$ using the notion of a dissociated sum and the result from \cite{temvcinas2021multivariate}. That normal approximation then gives a bound on the probability to be in $kn(n-1) \cdots (n-v(H_\ell) + 3) \pm 1/2$. 
}

{\paragraph{A  normal approximation} Here we derive a  normal approximation for $W=\sum_{i \in I(H,x,s)} Y(i)$ with 
$$Y (i) = \frac{X(i) - p_H}{\sigma}$$ 
where 
$$\sigma^2 = Var \left( \sum_{i \in I(H,x,s)} X(i) \right).$$
Then $Y(i)$ has mean zero, and 
$$Var  Y(i) = \frac{p_H ( 1- p_H)}{\sigma^2}. $$
To assess $\sigma^2$ we introduce the dependency neighbourhood of $Y(i)$ as  the set
$$D_s(i) = \{ j =  j(H,x)  \in I(H,x,s): \{  E(x) \setminus i_{H,x}(s) \} \cap \{ E(x) \setminus j_{H,x}(s) \} \ne \emptyset\}.$$
Then $W - \sum_{u \in D_s(i)} Y (u)$ is independent of $Y(i)$. Moreover, for $j \in D_s(i)$, the quantity 
$W - \sum_{u \in D_s(i)} Y (u) - \sum_{u \in D_s(j)\setminus D_s(i)} Y (u)$ is independent of $(Y(i), Y(j))$; we have a dissociated decomposition in the sense of \cite{temvcinas2021multivariate}. }
{Now, as observed in Remark 1.9 of \cite{reinert2019approximating}, 
for each injection there are $O(n^{v(H)-3})$ injections which share at least one edge, and hence
$$\sigma^2 = O( n^{2 v(H) -5}).$$
Employing  Theorem 3.3 in \cite{temvcinas2021multivariate} gives the following result.}
{
\begin{lemma}\label{lem:normal}
Let $Z$ be a standard normal variable. Then 
$$\sup_{ x \in \RR} | \PP ( \sum_{i \in I(H,x,s)} X(i) \le x) -
\PP (Z \le (x- | \wk{I} | p(H))/\sigma ) | = O(n^{-\frac12}).$$
\end{lemma}}

{ 
{\textbf{Proof of Lemma \ref{lem:normal}.}}}
{ 
Adapting Theorem 3.3 in \cite{temvcinas2021multivariate} to the univariate case gives that 
		 for any centered dissociated sum $W \in \RR$ with variance $\sigma^2>0$ and finite third absolute moments we have 
		\[
		\sup _{x}|\PP(W \le x )-\PP(\sigma^{\frac{1}{2}}Z \le x)| \leq 2^{\frac{7}{2}} 3^{-\frac{3}{4}}d^{\frac{3}{16}}B^{\frac{1}{4}},
		\]
		where 
  \[
B = B_1 + B_2\]
with 
\begin{align*}
B_1 &\coloneqq 
		\frac{1}{3}\sum_{s \in {I}} \sum_{t, u \in D(s)} \hspace{-.4em} \left(\frac{1}{2}\EE | {Y_sY_tY_u}|+\EE|{Y_sY_t} |\EE|{Y_u}|\right) \\
		B_2 &\coloneqq  \frac{1}{3} \sum_{s \in I} \sum_{t \in D(s)} \sum_{v \in D(t) \setminus D(s)} \hspace{-.4em} \left(\EE|{Y_sY_tY_v}| + \EE|{Y_sY_t} |\EE|{Y_v}|\right). 
		\end{align*}} 
{As $E(x) \setminus i_{H,x}(s)$ has $e(H) -1$ elements, we can bound the size of $D_s(i)$. To share a potential edge, it may suffice to share one vertex, which could then be connected to $S$.  
Thus, 
$|D_i(s)| \le C(H) n^{v(H)-3}$, where $C(H) $ is a constant which relates to the number of automorphisms of $H$. 
}
{Now using our bounds on $| D_i(s)|$ and $\sigma^2$ and noting that $v(H) \ge 3$, we obtain that
\[
B = O( n^{v(H)-2 } ( n^{v(H)-3})^2 (n^{-(v(H)-2.5)})^3 = O(n^{-\frac12}).
\]
Hence we have shown that $W$ is asymptotically standard normal. $\hfill \Box$}

{From the standard normal approximation Lemma \ref{lem:normal} we obtain that 
$$\PP \left( \sum_{i \in I(H,x,s)} X(i) \le x \right) =
\PP (W \le (x- | I | p(H))/\sigma ) $$
so that for $k = |I| p(H) + a \sigma$
$$\PP \left( \sum_{i \in I(H,x,s)} X(i) \le k\right) =
\PP (W \le a). $$
Hence, for all $\epsilon > 0 $ \begin{align*}
\PP \left( \sum_{i \in I(H,x,s)} X(i) =k \right)
& = 
    \PP \left( \sum_{i \in I(H,x,s)} X(i) \in \left(k - \epsilon, k+ \epsilon \right) \right) \\
    & \approx \Phi \left( \frac{k - | {I} |p(H) + \epsilon }{\sigma} \right) - 
\Phi \left( \frac{k - | {I} | p(H) - \epsilon }{\sigma} \right) \\
& \approx \frac{\epsilon}{\sigma \sqrt{2 \pi }} e^{ -\frac{a^2}{2}  }
\end{align*}
for $k = |I| p(H) + a \sigma$ and 
with $\Phi$ denoting the standard normal c.d.f.
As we can make $\epsilon$ as small as desired, this probability tends to 0;  
it goes to zero exponentially fast unless $k =   | I |p(H) + a \sigma$ for a fixed real number $a$.
Thus the only contribution is the approximation error and we have 
\begin{align*}
\PP \left( \sum_{i \in I(H,x,s)} X(i) =k \right) = \PP ( \Delta_s t (H,Z) = k) 
= O( \sigma^{-1} ) . \end{align*}} 

\medskip 
{\bf Different asymptotic regimes}

\medskip
{{\bf The case $N p(H,k)  \rightarrow 0.$}
Consider the case of counts of graphs $H$ with $p(H,k) = \PP (\Delta_s t(Z) = k) \rightarrow 0$ as $n \rightarrow \infty.$ For such $H$, for all $\epsilon>0,$ by Markov's inequality, 
\begin{eqnarray*}
    \PP \left( \left| \sum_{s \in [N]} \frac{ (Z_s - \astar) \mathbb{I} (\Delta_s t(Z) =k)}{1 + \sum_{u \ne s}  \mathbb{I} (\Delta_u t(Z) =k)} \right| > \epsilon \right) 
    &\le & \PP \left( \left| \sum_{s \in [N]} (Z_s - \astar) \mathbb{I} (\Delta_s t(Z) =k) \right| > \epsilon \right)\\
    &\le&  \frac{\EE\left| \sum_{s \in [N]} (Z_s - \astar) \mathbb{I} (\Delta_s t(Z) =k)\right| }{\epsilon} \\
&\le&  \frac{N p(H,k)}{\epsilon} \rightarrow 0 
\end{eqnarray*}
as $n\rightarrow \infty.$ Thus, for such $H$, asymptotic consistency follows.}

{In particular if $H$ has $v(H) \ge 5$ vertices and if $p(H,k) = \Theta(\sigma^{-1})$ then as $\sigma = \Theta(n^{v-\frac{5}{2}})$
we have  $N p(H,k) = \Theta (n^{9/2 - v(H)}) \rightarrow 0$ as $n \rightarrow \infty.$ Thus, for such $H$, asymptotic consistency follows. 
}

{Moreover, as subgraphs are required to be connected, any subgraph on at least 5 vertices, once $s$ is fixed, still has at least 3 free vertices. Subgraph counts can be expressed as functions of degrees, as we will exploit below. The degrees are binomially distributed and hence their probability mass is at most of the order $n^{-\frac12}$. Having at least 3 free vertices involves at least 3 such binomial distributions, which are only weakly dependent through the weak dependence in the degrees -- $\deg(i)$ and $\deg(j)$ depend on each other only through $Z_{i,j}$. Thus for graphs on at least 5 vertices we have $p(H,k) = n^{-\frac{3}{2}}$ and we are in the regime that $N p(H,k) \rightarrow 0$.}

\medskip
\gr{{\bf The case $ N p(H,k) = \Theta (n^{\frac{3}{2}}).$} In this case we write 
$$\frac{\sum_{s \in [N]} (Z_s - \astar) \I ( \Delta_s t = \uk) }{\sum_{u \in [N]} \I ( \Delta_u t = \uk)}= \frac{\frac{1}{N p(H,k)} \sum_{s \in [N]} (Z_s - \astar) \I ( \Delta_s t = \uk) }{\frac{1}{N p(H,k)}\sum_{u \in [N]} \I ( \Delta_u t = \uk)}.$$
First, by independence, 
$$\frac{1}{N p(H,k)} \EE  \sum_{s \in [N]} (Z_s - \astar) \I ( \Delta_s t = \uk) =0, \quad  \frac{1}{N p(H,k)} \EE \sum_{u \in [N]} \I ( \Delta_u t = \uk) = 1.$$
For the variance of the nominator, 
\begin{eqnarray*}
    \lefteqn{ Var \left( \frac{1}{N p(H,k)}\sum_{s \in [N]} (Z_s - \astar) \I ( \Delta_s t = k)\right) }\\ 
    &=& \frac{N p(H,k) ( 1 - p(H,k)}{N^2 p(H,k)^2} \\&& 
    + \frac{1}{N^2 p(H,k)^2}\sum_{s \in [N]}  \sum_{u \ne s} \EE (Z_s - \astar) (Z_u - \astar) \I ( \Delta_s t =k) 
    \I ( \Delta_u t =k) .
\end{eqnarray*}
For the second summand, let $Z_{-s}$ denote the collection of indicators $Z$ with $Z_s$ left out. Then by independence, 
\begin{eqnarray}
 \lefteqn{ \sum_{s \in [N]}  \sum_{u \ne s} \EE (Z_s - \astar) (Z_u - \astar) \I ( \Delta_s t =k) 
\I ( \Delta_u t =k) } \nonumber
\\ &=& 
\sum_{s \in [N]}  \sum_{u \ne s} \EE (Z_s - \astar) (Z_u - \astar) \I ( \Delta_s t =k) 
   \left\{  \I ( \Delta_u t =k) - \I (\Delta_u t(Z_{-s}) = k)  \right\} . \label{eq:var1}
\end{eqnarray}
Next we use Slutsky's Theorem. Starting with \eqref{eq:var1},  
the counts $\Delta_u t (Z) $ and $\Delta_u t(Z_{-s}) $ differ only if there is at least one occurrence of $H$ which includes both $s$ and $u$. On the event that $\Delta_s t =k $, there are at most $(e(H)-1) k$ edges which together with $s$ create a copy of $H$. Hence, 
\begin{eqnarray*}
 \lefteqn{ \frac{1}{N^2 p(H,k)^2} \left| \sum_{s \in [N]}  \sum_{u \ne s} \EE (Z_s - \astar) (Z_u - \astar) \I ( \Delta_s t =k) 
\I ( \Delta_u t =k) \right| }\\
&\le  & \frac{1}{N^2 p(H,k)^2}  \sum_{s \in [N]}  \sum_{u \ne s} \EE \I ( \Delta_s t =k) 
\I ( \Delta_u t =k) \I  ( \Delta_u t (Z)  \ne \Delta_u t(Z_{-s}) ) \\ 
&\le& \frac{N (e(H)-1) k p(H,k)}{N^2 p(H,k)^2}
\\
&=& \frac{ (e(H)-1) k }{N p(H,k)}.
\end{eqnarray*}
Using the normal approximation, $k = O(n)$, and the assumption $N p (H,k)  = \Theta ( n^{3/2})$, 
 $$ Var \left( \frac{1}{N p(H,k)}\sum_{s \in [N]} (Z_s - \astar) \I ( \Delta_s t = k)\right) = O(n^{-\frac12}).
$$
From Chebychev's inequality it follows that 
$$ \frac{1}{N p(H,k)}\sum_{s \in [N]} (Z_s - \astar) \I ( \Delta_s t = k) \rightarrow 0 $$ in probability. 
}

\gr{For the denominator we calculate the different cases for  $H$ one by one.}

\medskip
{{\bf The triangle case with $v=3$.} If $H$ is a triangle then $\Delta_s t(H)$ follows a binomial distribution with parameters $n-2$ and $(\astar)^2.$ From \eqref{eq:balasz} we have that the mode of the distribution has $p(H,k) = \Theta(n^{-\frac12}) $ so that indeed $N p(H,k) = \Theta (n^\frac32)$. 
Moreover, $\Delta_u t$ and $\Delta_s t $ are independent if $u$ and $s$ do not share a vertex. 
If they share a vertex then if $s=(i,j)$ and $u = i, j'$, with $\Delta_{u \setminus s} t $ denoting the count of $H$ including edge $u$ with edge  $s$ excluded. There is at most one triangle which includes both $s$ and $u$ if edges $s$ and $u$ are present, and this triangle is present with conditional probability $p$; excluding this triangle, the counts of $u$-triangles and $s$-triangles are independent as they do not share an edge. Hence 
\begin{eqnarray*}
  \EE \I ( \Delta_s t = k)  \I ( \Delta_u t = k) 
  &=& p \PP (\Delta_{s \setminus u} t = k-1) \PP (\Delta_{u \setminus s} t = k-1) \\
  &&+ (1- p) \PP (\Delta_{s \setminus u} t = k) \PP (\Delta_{u \setminus s} t = k) \\ 
  &=& p \PP (\Delta_{s \setminus u} t (H, Z) = (k-1)n)  \PP (\Delta_{u \setminus s} t (H, Z) = (k-1)n) \\
  &&+ (1- p) \PP (\Delta_{s \setminus u} t (H, Z) = kn) \PP (\Delta_{u \setminus s} t (H, Z) = kn).
\end{eqnarray*}
From \cite{balazs2014stirling} we know that 
   \begin{equation}
       \P (Bin (m,p) = k) = \frac{1}{\sqrt{2 \pi m p (1-p)}} e^{- \frac{(k - mp)^2}{ 2 m p (1-p)}} + O(m^{-1}) \label{eq:balasz}  
   \end{equation} 
and hence 
that for any $k$ the probability that a Binomial$(n,p)$ distribution takes on the value $0 < nk<  1$ is 
$$p(nk) = \frac{1}{\sqrt{n}} e^{-\frac{n(k-p)^2}{2 p (1-p)}} \frac{1}{\sqrt{2 \pi p (1-p)}} (1 + o(1)).$$
This gives that 
$ \EE \I ( \Delta_s t = k)  \I ( \Delta_u t = k) = O(n^{-1})$, of the same order as $p(H,k)^2.$ 
Concluding, 
\begin{eqnarray*}
    Var \left( \frac{1}{N p(H,k)}\sum_{s \in [N]}\I ( \Delta_s t = \uk)\right) 
    & \le &
\frac{1}{N^2 p(H,k)^2}\left( N p(H,k)  + N (n-2) O(n^{-1} ) \right) \\
&=&  O \left( \frac{1}{N p(H,k)^2}\right) \\
&=&  O \left( \frac{n}{N}\right) = O(n^{-1}).
\end{eqnarray*}
Hence, using Chebychev's inequality, 
$$
\frac{1}{N p(H,k)}\sum_{s \in [N]}\I ( \Delta_s t = \uk) \rightarrow 1 $$
in probability.
This is the last ingredient to show that our estimator is consistent for triangles.}

\medskip
{{\bf The 2-star case.} For $H$ being a 2-star we can employ a similar argument; for any fixed vertex pair $u$, the number  $\Delta_u t (H, Z)$ of 2-stars which would involve $u=(i,j)$ 
is the sum $\deg_{-j} (i) + \deg_{-i} (j)  $. Here $\deg_{-j} (i)$ is the number of neighbours of $i$ excluding the potential neighbour $j$. 
By independence, $\deg_{-j} (i) + \deg_{-i} (j) = \sum_{a \ne i, j}Z_{i,a}  + \sum_{b \ne i,j} Z_{j,b} $ has the Binomial distribution with parameters $2(n-2)$ and $\astar.$ Hence for $k$ close to the mean, $p(H,k) = \Theta (n^{-\frac12})$ by \eqref{eq:balasz}. Moreover, for $s=(i,j)$ and $u=(i',j)$, 
\begin{eqnarray*}
   \PP ( \Delta_u t = k, \Delta_s t=k) 
   &=& \PP ( \deg_{-j} (i) + \deg_{-i} (j) =k, 
   \deg_{-j'} (i) + \deg_{-i} (j') =k)\\
   &\le & \PP (\Delta_u t = k) 
   \end{eqnarray*}
  and the contribution to the variance is less or equal to 
  $$\frac{2 (n-2) N}{N^2 p(H,k)^2} p(H,k) 
  = O(n^{-\frac12}).$$
}

{ 
For $s \ne u$,  $\Delta_s$ and $\Delta_u$ are only weakly dependent through their degrees. Again we use Slutsky's Theorem. 
If $s$ and $u$ share a vertex then we can bound the covariance again as $O(n^{-1})$ and obtain an overall contribution to the variance of $O(n^{-1})$ as in the triangle case. If $s=(i,j)$ and $u=(i'.j')$ do not share a vertex, then $\deg_{-j} (i) = \I (i \sim i') + \I ( i \sim j') + \sum_{a \ne i.j,i',j'} Z_{a,i}$
where  $\sum_{a \ne i.j,i',j'} Z_{a,i}$ is Binomial $(n-4,p)$ and independent of $\deg_{-i} (j), \deg_{-i'} (j), $ and  $\deg_{-j'} (i').$ Conditioning on the different indicators gives a contribution to the covariance of order $n^{-1}$. However due to the global dependence, a finer argument is needed.}
{
Expanding gives 
\begin{eqnarray*}
 \lefteqn{  \PP ( \Delta_u t (Z) = k, \Delta_s t (Z) =k) }\\ 
   &=& \PP ( \deg_{-j} (i) + \deg_{-i} (j) =nk, 
   \deg_{-j'} (i') + \deg_{-i} (j') =nk)\\
    &=& \PP (
    \I (i \sim i') + \I ( i \sim j') + \sum_{a \ne i.j,i',j'} Z_{a,i}
    +
    \I (j \sim i') + \I ( j \sim j') + \sum_{a \ne i.j,i',j'} Z_{a,j}
    =nk, \\
    && 
   \I (i \sim i') + \I ( i' \sim j') + \sum_{a \ne i.j,i',j'} Z_{a,i'}
    +
    \I (j' \sim i') + \I ( j \sim j') + \sum_{a \ne i.j,i',j'} Z_{a,j'} =nk)
   \end{eqnarray*}
   We condition on the different outcomes. 
   }
   
   {
   First, if $ \I (i \sim i') = \I ( i' \sim j')= \I (i \sim j) = \I ( i' \sim j') =1,$
   \begin{eqnarray*}
 { (\astar)^4     \PP (
    \sum_{a \ne i.j,i',j'} Z_{a,i}
    +
    \sum_{b \ne i.j,i',j'} Z_{b,j}
    =nk -4, 
    \sum_{a \ne i.j,i',j'} Z_{a,i'}
    +
   \sum_{b \ne i.j,i',j'} Z_{a,j'} =k-4)}\\
   = (\astar)^4  \PP (Bin(2n-8, \astar) = nk-4)^2.
   \end{eqnarray*}
   Here we used the independence of the indicators. We compare this to 
   $\PP ( Bin (2n-4, \astar)^2 = nk).$  With \eqref{eq:balasz} and $m=2n-4, p = \astar,$
   \begin{eqnarray*}
     \lefteqn{   | e^{- \frac{(k - mp)^2}{ 2 m p (1-p)} } -
    e^{- \frac{(k +\epsilon  - (m + \delta) p)^2}{ 2 (m + \delta)  p (1-p)}}  |}\\
    &\le & \frac{1}{2 p (1-p) } \left|  \frac{(k+\epsilon - (m+\delta)p)^2}{(m+\delta)^2} - \frac{(k - mp)^2}{m^2} \right|\\
    &\le& \frac{1}{2 p (1-p) } \frac{| \epsilon m - k \delta}{m(m+\delta)} \left( \frac{k+\epsilon - (m+\delta)p}{m+\delta} + \frac{k - mp}{m} \right) .
   \end{eqnarray*}
   Here  for $k$ we take
   $kn = 2 (n-2) p + a \sqrt{n}$, and $\epsilon, \delta \in \{ 0 , 1, 2, 3, 4\}$, with $m = 2 (n-2),$ 
   and so this expression is $O(n^{-1})$. We can thus bound all probabilities by the same normal p.d.f., and  the contribution of each individual term to the covariance is $0 + \frac{1}{\sqrt{2  (2n-4) \astar (1 - \astar) } } O(n^{-1}) = O(n^{-3/2}).$ 
   The overall contribution to the covariance is $O( \frac{1}{p(H,k)^2} O (n^{-3/2}) = O(n^{-\frac12}).$ Using Chebychev's inequality yields consistency of $\tilde{p}(\uk,X)$ in this situation.
   }

\medskip\gr{{\bf The case that $v(H)=4$.} In this case, $\sigma^2 = O(n^{8-3}) = O(n^{\frac32})$.
 With \eqref{eq:var1},
\begin{eqnarray*}
 \lefteqn{ \sum_{s \in [N]}  \sum_{u \ne s} \EE (Z_s - \astar) (Z_u - \astar) \I ( \Delta_s t =k) 
\I ( \Delta_u t =k) } 
\\ &=& 
\sum_{s \in [N]}  \sum_{u \ne s} \EE (Z_s - \astar) (Z_u - \astar) \I ( \Delta_s t =k) 
  \I (  \Delta_u t \ne \Delta_u t(Z_{-s}) ) \\&& \left\{  \I ( \Delta_u t =k) - \I (\Delta_u t(Z_{-s}) = k)  \right\} . 
\end{eqnarray*}
Now $\E \Delta_s t (H,Z) = \Theta (n^{v-2}) = \Theta (n^{2})$ and so the simple overlap argument no longer works. Instead we distinguish the cases of whether or not $u$ and $s$ share a vertex. We have 
\begin{eqnarray*}
{\sum_{s \in [N]}  \sum_{u \ne s, | u \cap s | = 1} \EE (Z_s - \astar) (Z_u - \astar) \I ( \Delta_s t =k) 
    \I ( \Delta_u t =k)  }
  &\le& N n p(H,k) 
\end{eqnarray*}
giving a contribution to the variance of at most 
$$ \frac{N n p(H,k)}{N^2 p(H,k)^2} = \frac{1}{n p(H,k)}.$$
If $n p(H,k) \rightarrow \infty$ as $n\rightarrow \infty$ then this contribution will be small. 
} 

\gr{If $s$ and $u$ do not share an edge, then together they determine the set of vertices involved in any copy which hosts both $s$ and $u$.
Then  we condition on the different cases of which of these copies are present, deal with the independence as in the 2-star case, using that the difference between the normal probabilities in the approximation is of order $n^{-1}$, with the binomial probabilities of order $n^{-\frac12}$, and obtain a contribution of the order 
$$\frac{N^2 \frac{1}{n} \frac{1}{\sqrt{n}}\frac{1}{\sqrt{n}} }{N^2 p(H,k)^2}
= \frac{1}{( n p(H,k))^2}.$$
If $n p(H,k) \rightarrow \infty$ as $n\rightarrow \infty$ then this contribution will be small.}

\gr{Similarly, for the denominator of our estimator, 
\begin{eqnarray*}
\sum_{s \in [N]}  \sum_{u \ne s} Cov( \I ( \Delta_s t =k) ,
    \I ( \Delta_u t =k)  ) 
    &=& \sum_{s \in [N]}  \sum_{u \ne s, | u \cap s| =1} Cov( \I ( \Delta_s t =k) ,
    \I ( \Delta_u t =k)  ) \\
    && + \sum_{s \in [N]}  \sum_{u \ne s, | u \cap s | = 0} Cov( \I ( \Delta_s t =k) ,
    \I ( \Delta_u t =k)  ).
\end{eqnarray*} We can bound again
\begin{eqnarray*}
 \left| \sum_{s \in [N]}  \sum_{u \ne s, | u \cap s| =1} Cov( \I ( \Delta_s t =k) ,
    \I ( \Delta_u t =k)  ) \right| \le 
   2  N n p(H,k) 
\end{eqnarray*}
and similarly for the case $| s \cap u| = 0$ disentangle the few possible occurrences and the dependence to arrive at a similar variance bound as for the nominator, which tends to 0 if $n p(H,k) \rightarrow \infty.$
}

\medskip
\gr{{\bf The 3-star case.} If $v=4$ and $H$ is a 3-star, we can use a binomial argument as the number of 3-stars of a vertex $i$ is ${ {\deg (i)} \choose 2}$. Hence the number of 3-stars which a potential edge $s=(i,j)$ would be involved in is  ${ {\deg_{-j}  (i)} \choose 2} + { {\deg_{-i}  (j)} \choose 2}$. We can calculate $p(H,k)$ directly using that $\deg_{-j}  (i)$ and $\deg_{-i}  (j)$ are independent; 
\begin{eqnarray*}
    p(H,k)
    &=& \PP \left( { {\deg_{-j}  (i)} \choose 2} + { {\deg_{-i}  (j)} \choose 2} = k n(n-1) \right) \\
    &=& \sum_{\ell=0}^{\sqrt{k}  n} \PP \left( { {\deg_{-j}  (i)} \choose 2} = \ell \right) \PP \left( { {\deg_{-i}  (j)} \choose 2} = k  n(n-1) - \ell \right) 
    .
\end{eqnarray*}
Here we used that if $\deg_{-j}(i) > \sqrt{k}n$ then $\deg_{-j}(i) (\deg_{-j} -1) > kn(kn-1)$.
Now,
\begin{eqnarray*}
    \PP \left( { {\deg_{-j}  (i)} \choose 2} = \ell \right)
    &=& \PP \left( {\deg_{-j}  (i)} = \frac12 + \sqrt{ 2 \ell + \frac14} \right).
\end{eqnarray*}
With \eqref{eq:balasz},
\begin{eqnarray*}
  \lefteqn{  \PP \left( { {\deg_{-j}  (i)} \choose 2} = k n(n-1) - \ell \right)}\\ 
    &=& \PP \left( {\deg_{-j}  (i)} = \frac12 + \sqrt{ 2 (k n(n-1) - \ell + \frac14} \right)\\
    &=& \frac{1}{\sqrt{2 \pi (n-2) \astar (1-\astar)}} e^{- \frac{(\frac12 + \sqrt{ 2 (k n(n-1) - \ell + \frac14}  - (n-2) \astar)^2}{ 2 (n-2) \astar (1-\astar)}} + O(n^{-1})\\
    &=&\Theta(n^{-\frac12}) 
\end{eqnarray*}
for $\sqrt{2k} \approx (n-2) \astar$ and $\ell \approx \sqrt{k} n$, with the probabilities for $\ell$ more than $a \sqrt{n}$ away from $k (n-2)$ vanishing exponentially fast. Thus in this case, $p(H, k) = \Theta (n^{-\frac12}).$ 
Hence in this case the variance in \eqref{eq:var1} tends to zero and we can use Chebychev's inequality to show that  our estimator tends to 0 in probability.}

\medskip
\gr{ {\bf The case of a line on 4 vertices.} If $H$ is a line on 4 vertices then a copy which involves $s$ can arise either by $s=(i,j)$ being the middle edge, in which case the number of copies is $\deg_{-j}(i) + \deg_{-i} (j)$, or it can arise by $s$ being one of the outer edges. In the latter case the number of copies is 
$$\sum_{a=1}^{\deg_{-j}(i)} \sum_{b \ne i,j} Z_{a,b} +\sum_{c=1}^{\deg_{-i}(j)} \sum_{d \ne i,j} Z_{c,d}
$$ as each endpoint has its degree minus 1 as next edge choices, which in turn have their degrees as choices. We observe that the probability of having $k$ copies is $\Theta(n^{-\frac12}) $ for typical values of $k$, as then it follows already by the sum of probabilities that 
$p(H,k) = \Theta(n^{-\frac12})$. This is the condition needed for our above convergence argument to work, and hence we conclude convergence in probability.}

\medskip 
\gr{{\bf The case of a 4-cycle.} When $H$ is a cycle of size 4, then $\Delta_s t$ and $\Delta_u t$ share of the order of $n$ indicators jointly if they share a vertex. Moreover the normal approximation is  of order $n^{-\frac12}$ from our argument.
To calculate $p(H,k)$ in this case, let $s=(i,j)$. Then the number of 4 -cycles is 
$$ \sum_{a=1}^{n} \sum_{b=1}^{n} Z_{a,i} Z_{b,j} Z_{a,b} $$
with expectation $n^2 (\astar)^3.$ We could think of it has having $\deg_{-j}(i)$ choices for $a$ and $\deg_{-i}(j)$ choices for $b$, and then we require that $Z_{a,b}=1.$ 
We let $N_{-j} ( i) = \{ v \ne j: Z_{vi}=1\}$ denote the set of neighbours of $i$. If 
$V_{ij}:= N_{-j}( i) \cap N_{-i} (j)$ then the number of 4-cycles involving $s=(i,j)$, taking direction into account, is 
\begin{equation}
  2 \left(   \sum_{a\in  V_{ij}} \sum_{ b \ne a \in V_{ij} } Z_{a,b} + 
\sum_{a\in   N_{-j}( i) \setminus V_{ij}}\sum_{ b  \in N_{-i}( j) \setminus V_{ij} } Z_{a,b}  \right)\label{eq:4cycle}
\end{equation} 
and conditional on $a, b, $ and $v = | V_{ij}$
this expression is the sum of two independent binomial random variables, one $Bin\left( {v \choose 2}, p \right)$, and the other one $Bin((a-v)(b-v), p).$ This is equivalent to having one binomial random variable with distribution 
$Bin\left( {v \choose 2} + (a-v)(b-v), p \right)$. 
Moreover 
$\PP (N_{-j}( i) \setminus V_{ij} =a-v, N_{-i}( j) \setminus V_{ij} = b-v, V_{ij} =v)$
 follows a multinomial distribution having 4 groups (the last group corresponding to the vertices which are not connected to either $i$ or $j$). 
 We could think of this as each of the $n-2$ vertices connecting to $i$ but not $j$ with probability $\astar (1- \astar)$ and connecting to both with probability $(\astar)^2$, giving a multinomial distribution with parameters $(p_1, \ldots, p_4) = (\astar (1- \astar),\astar (1- \astar) ,(\astar)^2,(1- \astar)^2).$
The mode of a multinomial distribution is given in 
\cite{Finucan}. At a mode we have $a-v = b-v$  and, as an approximation, we can take $a-v, b-v$ close to  their expected values, $n \astar (1-\astar)$, and $v$ close to  $n (\astar)^2.$  
Employing \eqref{eq:balasz}, at this approximate  mode of the multinomial distribution, $Bin\left( {v \choose 2} + (a-v)(b-v), p \right)$ has pmf with maximum order $n^{-1}.$ 
Using that the binomial distributions for the degrees each peak at order $n^{-3/2}$ we obtain that $p(H,k)$ for $k$ close to the mean is of order $O((n^{-\frac12})^3)$.} 

\gr{We argue that in the case that $| s \cap u| =1 $ the probability that $\Delta_u t = \Delta_s t $ if $\Delta_s t =k $ is $O(p(H,k)^2) = O (n^{-3}).$ 
For this, assume that $u$ and $v$ share vertex $i$. Then in the above derivation the degree $\deg (i)$ is shared, but we still have the binomially distributed other degree, contribution to the order of $n^{-\frac12}$ to the probability. We would then have a multinomial distribution with more classes, distinguishing to which of the vertices $(i,j,j')$ a vertex $v$ connects. Arguing for the mode as before, distentangling the dependence gives a joint probability of the order $n^{-3}$. 
Then 
$$ \frac{1}{(N p(H,k))^2} \sum_{s \in [N]} \sum_{u \ne s, | u \cap s| =1} Cov  (\I (\Delta_u t =k) , \I ( \Delta_s t  =k) ) = O\left(\frac{N n n^{-3}}{N^2 p(H,k)^2} \right) = O (n^{-1}).$$
That still leaves the case $| s \cap u| = 0.$ In this case the above argument only yields $O(1)$; we need a finer bound on the correlation.} \gr{We can write out the count explicitly using \eqref{eq:4cycle}. To evaluate 
\begin{eqnarray*}
 \lefteqn{Cov(  \I (    \sum_{a\in  V_{ij}} \sum_{ b \ne a \in V_{ij} } Z_{a,b} + 
\sum_{a\in   N_{-j}( i) \setminus V_{ij}}\sum_{ b  \in N_{-i}( j) \setminus V_{ij} } Z_{a,b} =k), }\\
&& 
  \I (\sum_{a\in  V_{i'j'}} \sum_{ b \ne a \in V_{i'j'} } Z_{a,b} + 
\sum_{a\in   N_{-j'}( i') \setminus V_{i'j'}}\sum_{ b  \in N_{-i'}( j') \setminus V_{i'j'} } Z_{a,b} =k) 
\end{eqnarray*}
with the binomial construction the underlying random variables depend on each other only through potentially shared edges which are shared neighbours between $i, i'$, $j,j'$, $i,j'$, or $i',j$. These are themselves binomially distributed.  As above we can condition on these particular edges being present or not, leading to four random sums, each sum up to a binomial random variables, of random variables which differ from the original random variables in only a few edge indicators $A$. With $X_A = \{ X_a, a \in A\}$ with $|A|$ fixed, 
$$ \PP( \Delta_u t =k, \Delta_s t = k | X_A = x_A)  - p(H,k)^2 = O (n^{-1/2}p(H,k)^2). $$
Hence $$ \frac{1}{(N p(H,k))^2} \sum_{s \in [N]} \sum_{u \ne s, | u \cap s| =0} Cov  (\I (\Delta_u t =k) , \I ( \Delta_s t  =k) ) = O\left(\frac{N^2 n^{-1/2}}{N^2 } \right) = O (n^{-1/2}).$$
}

\gr{The case of a {\bf triangle-whisker graph} has again $p(H,k) = \Theta (n^{-1})$. To see this, if $s$ is the base of the graph then copies occur as binomially distributed $Bin(N(i,j), \astar)$, with  $N(i,j)$ the number  of joint neighbours of $i$ and $j$, which itself is $Bin (n-2, (\astar)^2)$; for typical values, the product of these probabilities are of order $n^{-1}$. If $s$ is a side edge of the triangle then given $N(i,j)$ the number of copies of H is $Bin (N(i,j) + 2, \astar)$, again giving a probability of order $n^{-1}$. If $s$ is the whisker of the graph then the number of copies of $H$ is the sum of two almost independent variables, one for each vertex in $s$; for vertex $i$ this variable has a conditional binomial distribution $Bin\left({ \deg_{-j} (i) \choose 2}, \astar\right)$, giving a probability of the order $n^{-3/2}$. Overall  $p(H,k)=O(n^{-1}).$}

\gr{
For the covariance, if $| u \subset s| = 1$ then there is only one free vertex to choose for creating a triangle-whisker copy which involves both $u$ and $s$, with $n-3$ choices. Conditioning on the edge indicators which create dependence gives a covariance contribution of the order $n^{-1/2} p(H,k)^2$ for each of these possible choices, so that the over contribution is of the order 
$O \left( \frac{N n n n^{-1/2} p(H,k)^2}{(N p(H,k))^2} \right) = O( n^{-1/2} )$.}

\gr{
If $| u \cap s| = 0$ then all 4 vertices are determined by $u$ and $s$. Conditioning on the edge indicators which create dependence gives a covariance contribution of the order $n^{-1/2} p(H,k)^2$ and an overall contribution of the order 
$O (n^{-1/2} )$. 
}

\gr{The case of a complete graph on 4 vertices with one edge missing, and the case of a complete graph on 4 vertices, follow similarly. For a complete graph with one edge missing, if $s$ is the base of the graph then the number of copies have the distribution $Bin\left({{n-2} \choose 2}, (\astar)^4\right)$, similarly for each side, so that $p(H,k) = \Theta (n^{-1}).$ 
For a complete graph, the number of copies involving $s$ has binomial distribution $Bin\left({{n-2} \choose 2}, (\astar)^5\right)$, so that $p(H,k) = \Theta (n^{-1}).$ We then argue as in the triangle-whisker case.}  

\medskip
\gr{This finishes the proof of the assertion.}
\end{proof}

\bigskip 
Proposition \ref{siconsistency2} shows that in edge-exchangeable graphs, 
$\hat{g}({\uk})$  consistently estimates $q(x^{(s,1)}) = q(x^{(s,1)}|{{\Delta_s t(x)}=\uk})$. 
In  an expanded version of Theorem \ref{consistenterop} we show that 
the approximate Stein operator from Eq.\eqref{eq:approx_stein},
$$
\A_{\widehat q ,{t} } f(x) := \frac{1}{N} \sum_{s \in [N]} \A_{\widehat q(x^{(s)}|{\Delta_s t(x)})} f(x),   
$$
with 
$$ {\widehat q (x^{(s,1)}|{\Delta_s t(x)}) } =  g({x_{-s}}), \quad 
{\widehat  q(x^{(s,0)}|{\Delta_s t(x)}) } =  1 - g({x_{-s}})$$ 
is a consistent estimator of
$$
\A_{ q, {t} } f(x) := \frac{1}{N} \sum_{s \in [N]} \A_{ q(x^{(s)}|{\Delta_s t(x)})} f(x)   
$$
{as $L \rightarrow \infty.$}
We recall
$$ {\AgraSSt}^2 (\hat q, t, x) 
=  \frac{1}{N^2} \sum_{s, s'\in [N]} h_x(s, s')
$$
with 
$$h_x(s, s') = \left\langle \A^{(s)}_{\hat q, {t} }  K(x,\cdot), \A^{(s')}_{\hat q , {t} }  K(\cdot,x)\right\rangle_{\H}.$$
We state the expanded version of Theorem \ref{consistenterop} here. 
 
\begin{theorem} \label{consistentopexpand}
{If the graph is edge-exchangeable then for}
any test function $f$ for which the
Stein operator $\A_{q, {t}  }f$ is well defined, and for all $\epsilon > 0$
$$ 
\mathbb{P}(|  \A_{\widehat q, {t} } f(X)  - \A_{q,  t} f (X)| > \epsilon) \le \frac{4 }{\epsilon^2 N^2  L  (|| \Delta f||)^{-2} }  {\left(  \sigma_n^2 (\uk) + \sigma_N^2 (\uk) \right).}
$$
Moreover, $ {\AgraSSt}^2 (\hat q, t, x) $
is a consistent estimator of 
$$ {\rm{gKSS}} (x) = \frac{1}{N^2} \sum_{s, s'\in [N]} 
\left\langle \A^{(s)}_{ q, {t} }  K(x,\cdot), \A^{(s')}_{ q, {t} }  K(\cdot,x)\right\rangle_{\H}.
$$
\end{theorem} 

\medskip 
\begin{proof}
We have that 
$$
\A_{ q, {t} } f(x) := \frac{1}{N} \sum_{s \in [N]} \A_{ q(x^{(s)}|{\Delta_s t(x)})} f(x).   
$$
and 
$$
\A_{\widehat q, {t} } f(x) := \frac{1}{N} \sum_{s \in [N]} \A_{\widehat q(x^{(s)}|{\Delta_s t(x)})} f(x),   
$$
with 
$$ {\widehat q (x^{(s,1)}|{\Delta_s t(x)}) } =  g({x_{-s}}), \quad 
{\widehat  q(x^{(s,0)}|{\Delta_s t(x)}) } =  1 - g({x_{-s}})$$ 
so that 
\begin{align*}
   \A_{\widehat q, {t} } f(x)  - \A_{q,  {t} } f (x) 
   &=  \frac1N \sum_{s \in [N]}  \{ ( g(\uk) - q(x^{(s)}|{\Delta_s t(x)} ))
   f(x^{(s,1)}) \\
    & \quad \quad \quad + ( 1 - g(\uk) - (1-q(x^{(s)}|{\Delta_s t(x)})) ) f ( x^{s,0}) \}
   \\&= \frac1N \sum_{s \in [N]}   ( g(\uk) - q(x^{(s)}|{\Delta_s t(x)})) \{ 
   f(x^{(s,1)}) - f ( x^{s,0}) \} \\
   &= \frac1N \sum_{s \in [N]}  ( g(\uk) - q(x^{(s)}|{\Delta_s t(x)})) \Delta_s f(x). 
\end{align*}
Hence 
$$
|  \A_{\widehat q, {t} } f(x)  - \A_{q, \gr{t} } f (x) | \le 
|| \Delta f|| \frac1N \sum_{s \in [N]} | g(\uk) - q(x^{(s)}|t({x_{-s}}))| . 
$$ 
With Proposition \ref{siconsistency2} {and using the edge-exchangeability},
\begin{align*}
    \mathbb{P} ( |  \A_{\widehat q, {t} } f(X)  - \A_{q, , {t} }  f (X) | > \epsilon) 
    &\le \frac{4  }{\epsilon^2 N^2 L  (|| \Delta f||)^{-2} }  {\left(\sigma_n^2 (\uk) + \sigma_N^2 (\uk) \right)}.
\end{align*}
The fact that taking the sup over functions in the Hilbert space $\mathcal{H}$ does not spoil the convergence follows from the closed form representation of the sup of ${\AgraSSt}^2$, see for example Equation (11) in \citep{xu2021stein}. 
We have that 
$$ {\AgraSSt}^2 (\hat q, t, x) 
=  \frac{1}{N^2} \sum_{s, s'\in [N]} h_x(s, s')
$$
where 
$$h_x(s, s') = \left\langle \A^{(s)}_{\hat q, {t} }  K(x,\cdot), \A^{(s')}_{\hat q, {t} }  K(\cdot,x)\right\rangle_{\H}.$$
Hence, 
\begin{eqnarray*}
\lefteqn{\AgraSSt^2 ({\widehat q} )  - \AgraSSt^2 (q) }\\
&=& 
\frac{1}{N^2} \sum_{s, s' \in [N] }
\left\langle \A^{(s)}_{{\widehat q, {t} } }  K(x,\cdot)  -  \A^{(s)}_{q, {t} }  K(x,\cdot), \A^{(s')}_{{\widehat q, {t} } }  K(\cdot,x)-  \A^{(s)}_{q, {t} }  K(x,\cdot) \right\rangle_{\H}
\end{eqnarray*} 
and the first part gives the desired convergence as $L \rightarrow \infty$.
\end{proof}


\subsection{Gaussian approximation for {{\AgraSSt} in} ERGMs}\label{sigauss}

Theorem \ref{consistentopexpand} shows that  the \AgraSSt \, operator is a consistent estimator for {the ERGM}  Glauber Stein operator.
{If the observed graph  $x$ is a realisation of an ERGM then results from \citet{xu2021stein} can be leveraged to obtain finer theoretical results.} 

{First we detail the scaling for exponential random graph models which is used in the theoretical results which follow.
For a graph $H$ on at most $n$ vertices 
$V(H)$ denote the vertex set, 
and for $x\in\{0,1\}^N$, denote by $t(H,x)$  the number of
{\it edge-preserving} injections from $V(H)$ to $V(x)$; an injection $\sigma$ preserves edges if for all edges $vw$ of $H$ {with  $\sigma(v)<\sigma(w)$}, $x_{\sigma(v)\sigma(w)}=1$.
For  $v_H =| V(H)| \ge 3$  set 
$$
t_H(x)=\frac{t(H,x)}{n(n-1)\cdots(n-v_H+3)}.
$$
If $H{=H_1}$ is a single edge, then $t_H(x)$ is twice the number of edges of $x$. In the exponent this scaling of counts matches Definition~1 in \citet{bhamidi2011mixing} and Sections~3 and~4 of \citet{chatterjee2013estimating}.
An ERGM {
for the collection  $x\in \{0,1\}^{N}$
can be defined
} as follows.
\begin{definition}[
Definition 1.5 in
\citet{reinert2019approximating}]
\label{def:ergm} 
Fix $n\in \N$ and $k \in \N$. {{L}et $H_1$ be a single edge {and f}or $l={2}, \ldots, k$ let}  $H_l$ be a connected graph on at most $n$ vertices; 
set $t_l(x) = t_{H_l}(x)$. For  $\beta = (\beta_1, \dots, \beta_k)^{\top} {\in \R^k}$ 
and
$t(x) =(t_1(x),\dots,t_k(x))^{\top} \in \R^k$ 
$X\in \G^{lab}_n$ follows  the exponential random graph model  $X\sim \operatorname{ERGM}(\beta, t)$ if for  $\forall x\in \G^{lab}_n$,
  $$  q(X = x) = \frac{1}{\kappa_n(\beta)}\exp{\left(\sum_{l=1}^{k} \beta_l t_l(x) \right)}.$$
Here $\kappa_n(\beta)$ is a  normalisation constant.
\end{definition}
} 

In particular, 
under suitable conditions,  the  ERGM Glauber Stein operator is close to the  $G(n,p)$ Stein operator. This result is already shown in \citet{reinert2019approximating}, Theorem 1.7, with details provided in the proof of Theorem 1 in  \citet{xu2021stein}. To give the result, a technical assumption is required, which originates in 
 \citet{chatterjee2013estimating}, and is required in   \citet{reinert2019approximating}. 
{For $a\in[0,1]$, define the following}
functions \citep{bhamidi2011mixing,eldan2018exponential}, 
{with the notation in Definition~\ref{def:ergm}  for ERGM$(\beta, t)$}:
$$\Phi(a) := \sum_{l=1}^{k} \beta_l e_l a^{e_l -1}, \quad
\varphi(a) 
:= \frac{1+ \tanh(\Phi(a))}{2} 
$$
where $e_l$ is the number of edges in $H_l$.

\begin{Assumption}\label{assum:er_approx}
$\operatorname{(1)}$ $\frac{1}{2}|\Phi|'(1) < 1$.
$\operatorname{(2)}$ $\exists a^{*} \in [0,1]$ that solves the equation $\varphi(a^{*}) = a^{*}$.
\end{Assumption}
{The value $a^*$ will be the edge probability in the approximating Bernoulli random graph, $\operatorname{ER}(a^*)$. {The following result holds.} 

\begin{proposition} \label{ergmconvergence} 
Let $q(x)=\operatorname{ERGM}(\beta, t)$ {satisfy} Assumption \ref{assum:er_approx} and let ${\tilde q}$ denote the distribution of {ER$(a^*)$.}
Then there is an explicit constant $C=C(\beta, t, {K})$  such that for all $\epsilon > 0,$ 
$$  \frac{1}{N} \sum_{s \in {N}} \E | (\A_q^{(s)}  f(Y) - \A_{\tilde{q} } ^{(s)} f(Y) ) | \le || \Delta f|| {n \choose 2} \frac{C(\beta, t)}{\sqrt{n}}.$$
Moreover, for  $f\in \H$ equipped with kernel $K$, let
$
f_x^*(\cdot) = \frac{ (\A_q - \A_{\tilde q} ) K(x,\cdot)}{\left\|(\A_q - \A_{\tilde q} ) K(x,\cdot) \right\|_{\H}}.
$
Then there is an explicit constant $C=C(\beta, t, {K})$  such that for all $\epsilon > 0,$
\begin{eqnarray}
\lefteqn
{
\P ( |  \gKSS (q,X)  - \gKSS ( {\tilde q}, Y) | \, >  \,  \epsilon)}\\
&\le& \Big\{  || \Delta (\gKSS(q, \cdot))^2 || ( 1 +  || \Delta \gKSS(q, \cdot) || ) +  4 \sup_x 
(|| \Delta f_x^*||^2)  \Big\} 
{n \choose 2}  \frac{C}{{\epsilon^2 \sqrt{n}}} . \label{ineq:bound}
\end{eqnarray}
\end{proposition}

\begin{proof}
 {The assertion follows immediately from  the proof of 
Theorem 1 in \citet{xu2021stein}.}
\end{proof}

The approximation with a Bernoulli random graph is useful as for a Bernoulli random graphs a normal approximation for its $\gKSS$ is available in \citet{xu2021stein}, under suitable assumptions. 

\begin{Assumption}\label{assum:er_approx_kernel}
Let $\H$ be the RKHS 
associated
with the  kernel $K: { \{ 0, 1 \}^N} \times  { \{ 0, 1 \}^N}\to \R$ and  for $s\in [N]$ let  $\H_s$ be the RKHS 
associated with the kernel $l_s: \{ 0, 1 \} \times  \{ 0, 1 \} \to \R$. Then \vspace{-2mm} 
\begin{enumerate}[i)]
    \item \label{ass21} $\H$ is a tensor product RKHS, $\H = \otimes_{s \in [n]}\H_s $; \vspace{-2mm} 
    \item \label{ass22} $k$ is a product kernel, $k(x, y) = \otimes_{s \in [N]} l_s(x_s, y_s)$; \vspace{-2mm} 
    \item \label{ass23} $ \langle l_s (x_s, \cdot), l_s (x_s, \cdot) \rangle_{\H_s}  =1$; \vspace{-2mm} 
    \item \label{ass24} $l_s(1, \cdot) - l_s(0, \cdot) \ne 0$ for all $s \in [N]$. \vspace{-2mm} 
\end{enumerate}
\end{Assumption}
These assumptions are satisfied for example for the suitably standardised Gaussian kernel $K(x,y) = \exp\{ - \frac{1}{\sigma^2} \sum_{s \in [N]} (x_s - y_s)^2\} $.

Letting 
$|| \cdot ||_1$ denote  $L_1$-distance, and $\mathcal L$ denote the law of a random variable, \citet{xu2021stein} show  the following normal approximation.}

{ \begin{theorem}[Theorem 2 in \citet{xu2021stein}] \label{th:normalapprox}
{Let $Y$ have the distribution $\tilde{q}$ of a Bernoulli random graph $ER(a^*)$ as in \cref{ergmconvergence}.}
 Assume that the conditions i) - iv) in Assumption \ref{assum:er_approx_kernel} hold. 
 Let $\mu = \E [  \operatorname{gKSS}^2 ( {\tilde q}, Y)] $ and $\sigma^2 = Var [ \operatorname{gKSS}^2 ( {\tilde q}, Y)]. $ Set 
 $W = \frac{1}{\sigma} (  \operatorname{gKSS}^2 ( {\tilde q}, Y)]  - \mu)$ and let  $Z$ denote a standard normal variable, Then there is an explicit constant $C = C(a^*, l_s, s \in [N]) $ such that 
\begin{equation*} 
    || {\mathcal L}(W) - {\mathcal L}(Z)||_1  \le \frac{C}{\sqrt{N}}. 
    \end{equation*} 
 \end{theorem}
{Thus a normal approximation for the approximating \gKSS{} can then be used to assess the theoretical behaviour of \AgraSSt{} {as follows.} 

\begin{corollary}\label{cor:Gaussian approximation}
Let the assumptions  \cref{ergmconvergence} and Theorem \ref{th:normalapprox} be satisfied. With the notation of \cref{th:normalapprox}, 
assume that the RKHS kernel $K$ is such that the right hand side of \cref{ineq:bound} is  
$ o(n).$ Then 
$\frac{1}{\sigma} ({\AgraSSt} ({\widetilde q}(x^{(s)}|{\Delta_s t(x)}) ) - \mu)$ is approximately standard normally distributed as $N \rightarrow \infty.$
\end{corollary}
\begin{proof}
For all $\epsilon > 0$, 
\begin{eqnarray*}
\lefteqn{
\mathbb{P} \left[ \left| \AgraSSt ({\widetilde q}(x^{(s)}|{\Delta_s t(x)}) ) - \gKSS ( {a^*} ) \right| > \epsilon\right] }
\\& \le &
\mathbb{P}  \left[ \left| \AgraSSt ({\widetilde q}(x^{(s)}|{\Delta_s t(x)}) )  - \gKSS (q) \right| >  \frac12 \epsilon \right] + \mathbb{P}  \left[ \left|\gKSS(q)   -  \gKSS ( {{a^*}}) \right| > \frac12 \epsilon \right] . 
\end{eqnarray*}
The first summand tends to 0 as $N \rightarrow \infty$ due to \cref{th:consist} and the second summand tends to 0 due to  \cref{ergmconvergence}. 
That $\gKSS ( {{a^*}})$ is approximately normally distributed with the appropriate scaling follows from \cref{th:normalapprox}. 
\end{proof}

The theoretical behaviour of the  subsampling version {$\widehat{\AgraSSt} ({\widetilde q}(x^{(s)}|{\Delta_s t(x)}) )$} is addressed in Proposition \ref{Bconsistency}.} {A detailed examination of the choice of kernel $K$ such that the assumptions of \cref{cor:Gaussian approximation}  are satisfied is left for future work.}

\section{{Additional background}}
In this section, we present {additional background to} 
complement the discussions in the main text.

\subsection{Parameter estimation for random graphs}\label{app:param-estimation}
Estimating parameters for parametric models is possible \emph{only} when the parametric family is \textit{explicitly specified}. For instance, in the synthetic example for  E2ST model shown in Section \ref{sec:synthetic_exp}, 
$\hat \beta_l$ can be estimated for $\beta_l$ since the edge, 2Star and triangle statistics are specified. There are various approaches for parameter estimation.

\paragraph{Maximum likelihood} {Maximum likelihood is}  
a popular approach for parameter estimation in random graph models. {A complication arises because its probability mass function from}
Eq.\eqref{eq:ergm},
\begin{equation*}
    q(X = x) = \frac{1}{\kappa_n(\beta)}\exp{\left(\sum_{l=1}^{k} \beta_l t_l(x) \right)}.
\end{equation*}
{involves a} normalisation constant $\kappa_n(\beta)=\sum_x \exp \{\sum_{l=1}^k \beta_l t_l(x)\}$ {which is g}enerally intractable  and needs to be estimated for performing MLE. {For this task,} 
Markov chain Monte-Carlo maximum likelihood estimat{ion} (MCMCMLE) for ERGM has been developed by \citet{snijders2002markov}.
When the network size is large, 
{accurate} estimation for the normalised $\kappa_n(\beta)$ requires large amount of Monte-Carlo samples {and is hence computationally expensive.}

\paragraph{Maximum pseudo-likelihood estimator} 
To alleviate the problem associated with {the normalising constant}, 
{Maximum Pseudo-likelihood Estimation (MPLE)} \citep{besag1975statistical} has been developed for ERGMs, {see \citet{strauss1990pseudolikelihood} and also  \citet{schmid2017exponential}}. MPLE {factorises} 
the conditional edge probability to approximate the exact likelihood, 
\begin{equation}\label{eq:mple}
q(x) = \Pi_{s\in[N]} q(x^{s}|x_{-s}).
\end{equation}
For {ERGMs}
the conditional distribution $q(x^{s}|x_{-s})$ does not {involve the normalising constant and can hence be computed more efficiently than the MLE.} 
However, {in general the MPLS is not consistent for ERGMs} 
as the edges are generally non-independent. The consistency of MPLE for Boltzmann machines is shown in \citet{hyvarinen2006consistency}.
{A thorough comparison of MCMCMLE and MPLE estimation in ERGMs can be found in \citet{van2009framework}}. 

\textbf{Contrastive divergence}
Estimation based on {contrastive divergence (CD)} \citep{hinton2002training} has also been developed for ERGM estimation \citep{hunter2006inference}. Contrastive divergence runs {a small number of} Markov chains simultaneously for $T$ steps and estimates the gradient 
{based  on the differences  between initial values and values after $T$ steps} in order to find a maximum}. {Convergence results}
for exponential family models are shown in \citet{jiang2018convergence}. CD {can provide}  a useful balance between computationally expensive {but} accurate MCMCMLE and fast {but} inconsistent MPLE.

\subsection{{Kernel Stein discrepancies and kernel-based nonparametric hypothesis testing}} 
\label{sec:ksd}

The task of hypothesis testing involves the comparison of distributions $p$ and $q$ that are significantly different with respect to the size of the test, denoted by $\alpha$. 
{In nonparametric tests,} the
distributions 
are not assumed to be in any parametric families {and test statistics are often based on ranking of observations}.
In contrast, parametric tests, such as a Student t-test or a normality test, assume a pre-defined parametric family to be tested against and usually employ a 
particular summary statistics such as means or standard deviations.
Recent advances in nonparametric test procedures introduce RKHS functions which can be rich enough to distinguish distributions whenever they differ. {Below we detail two instances which are relevant for the main paper.}


We start with a terse review of {\bf  kernel Stein discrepancy (KSD) for continuous distributions}
developed to compare and test distributions \citep{gorham2015measuring,ley2017stein}.
Let $q$ be a smooth probability density on 
$\mathbb{R}^d$ that vanishes at the boundary.  The operator $\T_q: (\mathbb{R}^d \to \mathbb{R}^d) \to (\mathbb{R}^d \to \mathbb{R})$ is called a {\it Stein operator} if the following {\it Stein identity} holds: $\E_q[{\T}_q f] = 0$, where $f
:\mathbb{R}^d \to \mathbb{R}^d$ is any bounded smooth function. 
A  {suitable 
function class  ${\mathcal F}$ is such that if} 
 $\E_p[{\T}_q f] = 0$ for all functions $f { \in {\mathcal F}}$, then $p=q$ follows. 
{It is convenient to take 
${\mathcal F} = B_1( {\mathcal H })$,} the 
unit ball of {a} {large enough} RKHS { with bounded kernel $K$}.
{T}he kernel Stein discrepancy (KSD) between two densities $p$ and $q$ {based on $\T_q$} is defined as
\begin{equation}
\operatorname{KSD}(p\|q, {{\mathcal H }}) =\sup_{f \in B_1(\mathcal H)} \mathbb{E}_{p}[{\T}_q f]. 
\label{eq:ksd}
\end{equation}
{Under mild regularity conditions, {for a particular choice of $\T$ called  Langevin operator,}}  $\mathrm{KSD}(p\|q, {{\mathcal H }}) \geq 0$ and $\mathrm{KSD}(p\|q, {{\mathcal H }}) = 0$ if and only if $p=q$
\citep{chwialkowski2016kernel}, {in which case} 
KSD
is a 
proper
discrepancy measure between probability densities. 

{The KSD in  Eq.\eqref{eq:ksd} can be used to test the model goodness-of-fit as follows.} 
{One can show that}
$\operatorname{KSD}^2(p\|q, {{\mathcal H }}) = {\E}_{x,\tilde{x} \sim p} [h_q(x,\tilde{x})]$,
where {$x$ and $\tilde{x}$ are independent random variables with density $p$} and  $h_q(x,\tilde{x})$ {is given in explicit form {which does not involve $p$},
\begin{eqnarray} 
h_q(x,\tilde{x})&=& 
 \left\langle \T_{ q} K(x,\cdot), \T_{q} K(\cdot,\tilde{x})\right\rangle_{\H}. \label{hform} 
\end{eqnarray}}  

Given a set of samples $\{ x_1,\dots,x_n \}$ from an unknown density $p$ on $\mathbb{R}^d$, to test whether
$p=q$, the statistic 
 $\mathrm{KSD}^2(p\|q, {{\mathcal H }})$ can be empirically estimated by {independent} samples from $p$ using a $U$- or $V$-statistic.
The critical value is determined by  bootstrap based on
{weighted chi-square approximations for} 
$U$- or $V$-statistics. 
%
{For goodness-of-fit tests of discrete distributions} {when i.i.d.\,samples are available}, 
a kernel discrete Stein discrepancy (KDSD) has been proposed {in} \citet{yang2018goodness}.

\textbf{Goodness-of-fit Testing} aims to check the null hypothesis  $\nullH_0: p=q$ against the {general} alternative $\nullH_1: p\neq q$ when the target distribution $q$ is explicitly specified.
Given sample(s) from {the} \emph{unknown} distribution $p$ and {an} explicit density $q$, {$\nullH_0$ is assessed using} a chosen test statistic, usually a discrepancy measure, $D(q\|p)$, between $p$ and $q$, which can be estimated empirically.
{Kernel-based hypothesis tests on goodness-of-fit {for continuous distributions $q$} use the kernel Stein discrepancy (KSD) in Section \ref{sec:ksd} as the test statistic. 
Given samples $x_1,\dots,x_n$ from  the \emph{unknown} density $p$,  $\operatorname{KSD}^2(p\|q, \H)$ in Eq.\eqref{eq:ksd} is estimated via the $V$-statistic
$$
\operatorname{\widehat{KSD}}^2(p\|q, {{\mathcal H }}) = \frac{1}{n^2}\sum_{i,j} h_q(x_i,{x_j});
$$
recall that $h_q(x_i,{x_j}) = \left\langle \T_{ q} K(x_i,\cdot), \T_{q}K({x_j},\cdot)\right\rangle_{\H}$ from Eq.\eqref{hform}.
The null distribution {of this test statistic} involves integral operators that are not available in close form; {often it is simulated using a}
wild-bootstrap procedure \citep{chwialkowski2014wild}.} With the (simulated) null distribution, the critical value {of the test} can be {estimated}
to decide whether the null hypothesis is rejected at test level $\alpha$.
In this way, a general method 
{for nonparametric testing of}  
goodness-of-fit 
on $\mathbb{R}^d$ is obtained, which is applicable even for 
models {with an intractable normalising constant}.

\textbf{Two-sample Testing} aims to determine whether two sets of samples are drawn from the same distribution, i.e. instead of 
$q$ {being available} in density form as in the goodness-of-fit setting, $q$ is only accessible  {through samples}.
Maximum mean embedding (MMD) test are {often} used for {this} two-sample problem \citep{gretton2007kernel}. {These tests are}
based on the kernel mean embedding of {a} distribution,
\begin{equation}\label{eq:mean_embedding}
    \mu_p:= \E_{x\sim p}[k(x,\cdot)] = \int_\X k(x,\cdot)dp(x) \in \H,
\end{equation}
whenever $\mu_p$ exist. Similar to KSD, MMD takes the {supremum} over unit ball RKHS functions;
\begin{equation}
    \label{eq:mean_embedding_distance}
    \operatorname{MMD}(p\|q) = \sup_{f \in B_1(\mathcal H)} \big| \E_p[f] - \E_q[f] \big| = \|\mu_p - \mu_Q\|_{\H}.
\end{equation}
With samples $x_1,\dots x_m \sim p$ and $y_1,\dots,y_n \sim q$, MMD can be estimated empirically via $U$-statistics,
\begin{align}\label{eq:mmd_u}
{\widehat{{\operatorname{MMD}^2_u}}}(p\|q) = \frac{1}{\small m(m-1)}\sum_{i\neq i'}k(x_i,x_{i'}) + \frac{1}{n(n-1)}\sum_{j \neq j'}k(y_{j},y_{j'}) - \frac{2}{mn}\sum_{i j}k(x_i,y_{j}).
\end{align}
In such kernel-based two-sample tests, the null distribution can be obtained via a permutation procedure \citep{gretton2007kernel};
{this procedure}
can be more robust compared to a wild-bootstrap procedure, especially when the kernels need to be optimised \citep{gretton2012optimal,jitkrittum2016interpretable, liu2020learning, liu2021meta}. 

The two-sample procedure can also be applied to 
{verify}  model assumptions when the model is not directly accessible through its distribution but through generated samples. Such a strategy has been considered as benchmark testing procedure in various studies for goodness-of-fit tests
\citep{jitkrittum2017linear, xu2020stein, xu2021interpretable}. Despite lower test power compared to the corresponding state-of-the-art KSD-based tests and higher computational cost due to additional empirical estimation for the distribution $q$, the MMD-based tests are competitive with a simpler derivation in complicated testing scenarios \citep{xu2020stein,xu2021interpretable}, {and they can} outperform non-kernel based goodness-of-fit tests as discussed in \citet{xu2020stein}.

\section{Visual illustrations {of}
the assessment procedures}\label{app:illustration}

{While  \AgraSSt{}} 
is illustrated in Figure.\ref{fig:florentine}, we provide an additional visualisation emphasising different {tasks for which \AgraSSt{} can be applied.}
In \cref{sec:assessing} in the main text, we mentioned two features of our proposed AgraSSt procedure: 
\begin{enumerate}
    \item Regardless of the learning or training procedures (masked in grey), \AgraSSt{} can test a given generator $G$ that is only accessible through its generated samples as shown in Figure.\ref{fig:florentine_sim}. In this setting, we do not need to know how the generator $G$ is obtained and the focus is the assessment of a particular generator $G$ itself. 

\begin{figure}[h]
    \centering
    \vspace{-.5cm}
    \includegraphics[width=1.\textwidth]{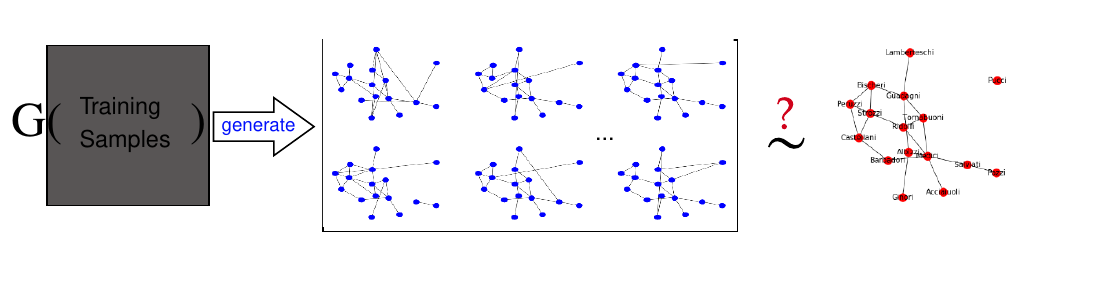}
    \vspace{-1.2cm}
    \caption{Assessing trained graph generators.}
    \label{fig:florentine_sim}
\end{figure}
\item Moreover, we are also interested in understanding the quality and capability of training procedures of (deep) generative models. 
As illustrated in Figure.\ref{fig:florentine2}, a generator $G$ is trained from the same distribution as the input graph, e.g. ERGMs. The focus in this setting is to assess the training procedure of the generative model. (The samples generated are masked in grey.)
For instance, for $G$ trained from the Florentine marriage network \citep{padgett1993robust}, we {may like to}
understand whether the generative model can be trained to generate graphs that resemble the Florentine marriage network.

\begin{figure}[h]
    \centering
    \vspace{-.5cm}
    \includegraphics[width=1.\textwidth]{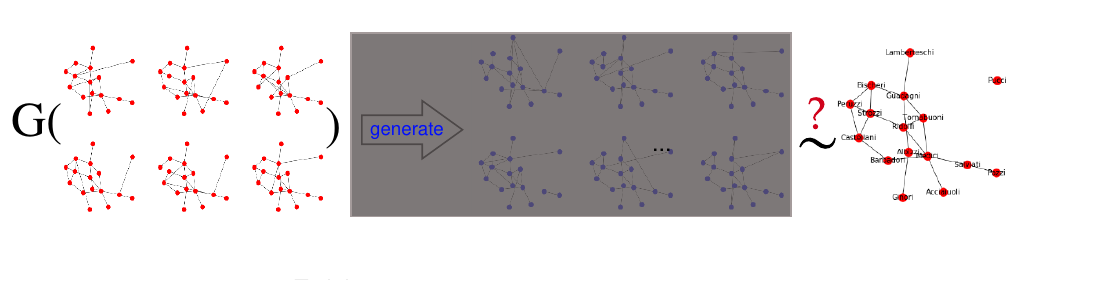}
    \vspace{-1.cm}
    \caption{Criticising training quality for generative models.}
    \label{fig:florentine2}
        \vspace{-.5cm}
\end{figure}

\end{enumerate}

\section{Additional experimental results and discussions}\label{app:exp}

\subsection{Generating reliable samples}

{To illustrate how \AgraSSt{} can be used to select sample batches, Figure.\ref{fig:karate} shows three sample batches of size 8 for the Karate club network of \cite{zachary1977information}, {including the  corresponding $p$-values for the displayed sample batches.} Here we would expect to detect some community structure in the networks; only the sample batch from CELL captures this feature at least to some extent {and has $p$-value which would not lead to rejection at the 5\% level}. This finding chimes with the results from Table \ref{tab:karate_rej}; AgraSSt rejects both GraphRNN and NetGAN as synthetic data generators, but does not reject CELL.}  

\begin{figure}[t!]
    \centering
    \subfigure[The Karate Club network (vertices in red)]{\includegraphics[width=0.39\textwidth]{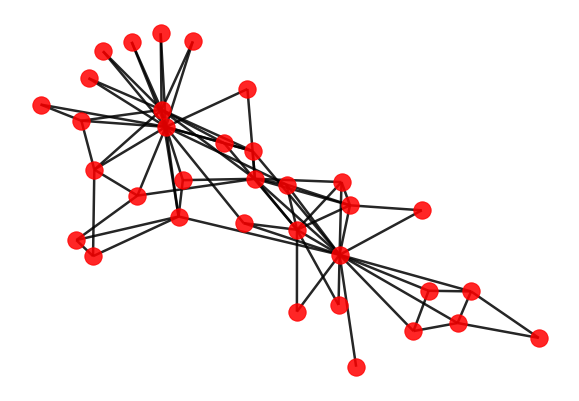}
    \label{fig:karate_club}}
    
    \vspace{.8cm}
    \subfigure[Samples generated from GraphRNN model trained on Karate Club network (vertices in green)]{\includegraphics[width=0.87\textwidth]{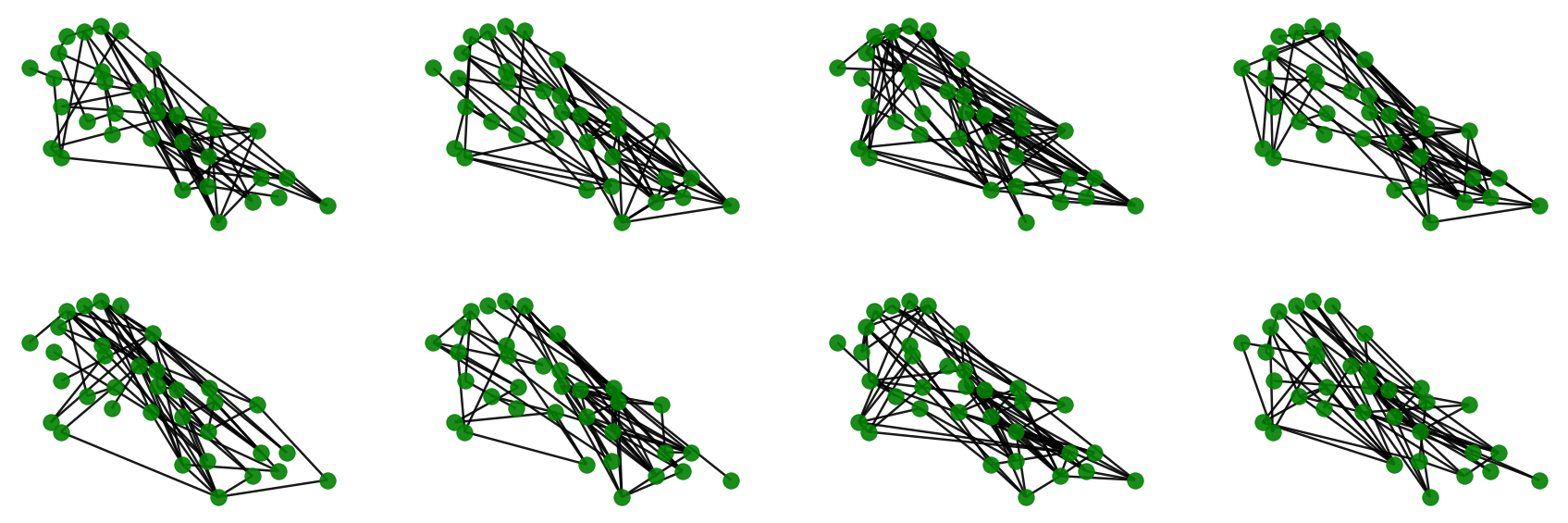}\label{fig:karate_graphrnn}}

    \vspace{0.25cm}
        \subfigure[Samples generated from NetGAN model trained on Karate Club network (vertices in orange)]{\includegraphics[width=0.87\textwidth]{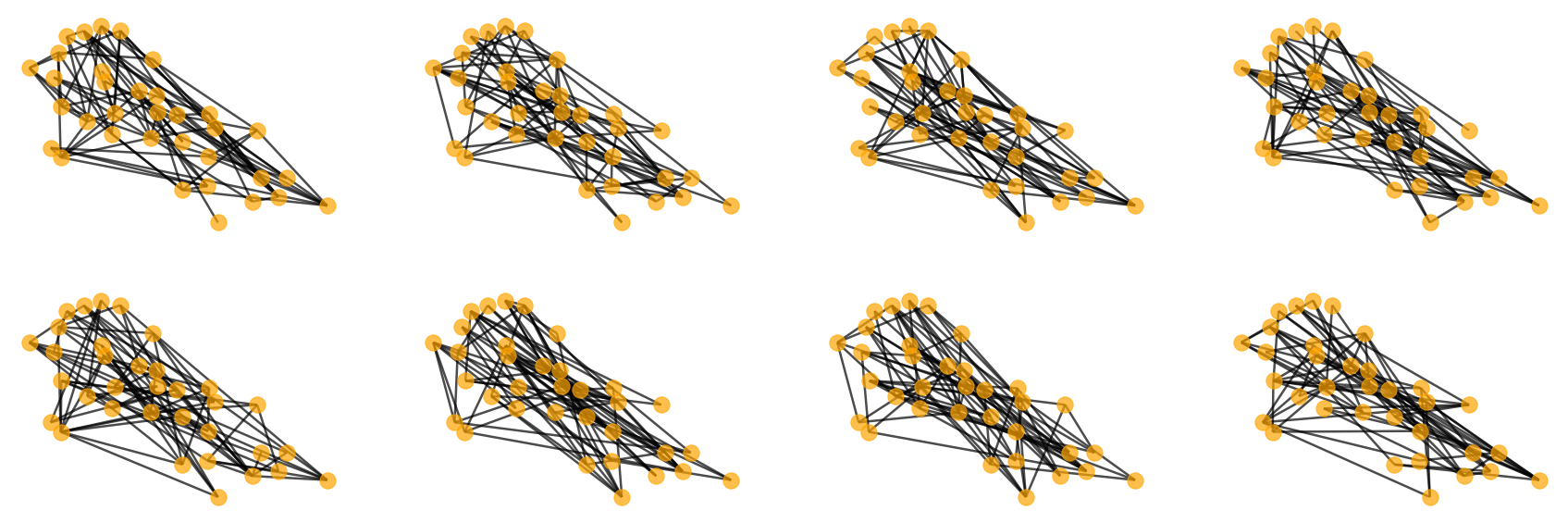}\label{fig:karate_netgan}}
    
    \vspace{0.25cm}
        \subfigure[Samples generated from CELL model trained on Karate Club network (vertices in blue) 0.26]{\includegraphics[width=0.87\textwidth]{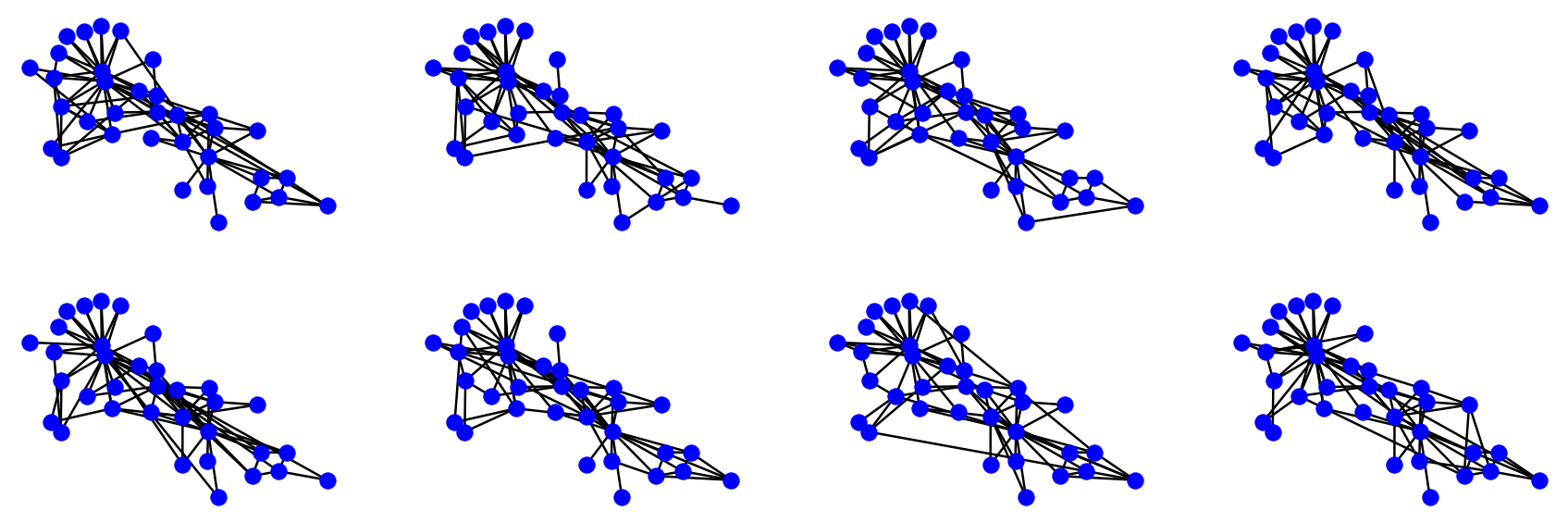}\label{fig:karate_cell}}
    \caption{The  Karate Club network \cite{zachary1977information} and three sample batches of size 8 from different graph generators. The $p$-value for GraphRNN samples in (b) is $0.00$, for  NetGAN samples in (c) the $p$-value is $0.01$; for CELL samples in (d) the $p$-value is $0.26$. 
    }
    \label{fig:karate}
\end{figure}

\subsection{Additional case study: Padgett's Florentine network} \label{app:karate}

Padgett's Florentine network \citep{padgett1993robust}. 
{
has 16 vertices and 20 edges; in \citet{xu2021stein} the hypothesis that it is an instance of a $G(n,p)$
model could not be rejected.
} 

\begin{table}[htp!]
    \centering
    \begin{tabular}{c|ccccc}
        \toprule
        {} & AgraSSt &  
        Deg & 
        MDdeg & 
        TV\_deg \\
        \midrule
         GraphRNN & {\color{red}0.01} & 0.11 & 0.26 & {\color{red}0.03} \\
         NetGAN & 0.16 & 0.18 & 0.09  & 0.06 \\
         CELL & 0.23  & 0.36 & 0.69 & 0.18 \\
        \bottomrule
    \end{tabular}
\vspace{0.2cm}
    \caption{$p$-values for models trained from Florentine marriage network; 100 samples to simulate the null; rejected null at significant level $\alpha=0.05$ is marked red.}
    \vspace{-0.1cm}
    \label{tab:florentine_rej}
\end{table}

The $p$-values for different tests are  shown in Table.\ref{tab:florentine_rej}. 
The Florentine marriage network has edge density $q=0.167$, while the trained CELL has $\q=0.165$ which  is a close approximation. GraphRNN generates graphs with higher edge density $\q=0.188$. NetGAN generate samples with $\q=0.176$,  
not too different from the null, which is not rejected at $\alpha=0.05$. This is different from what we see in the ERGM case above. {This discrepancy may arise as}  the Florentine network is small with $n=16$ and not highly clustered, with average {local}  clustering coefficient $0.191$.

\paragraph{{Sample batch selection}}
With CELL being {deemed} a good generator for the Florentine {marriage}  network, we generate a sample batch of size $30$ and check the sample quality. Most sample batches produce a $p$-value above $\alpha=0.05$ until the 8th batch, which  has $p$-value $0.03 <\alpha$. {\AgraSSt{} would recommend not taking this batch.} 
A visual illustration is shown in Figure.\ref{fig:batches}.

{To investigate these batches, we note that t}he Florentine marriage network has $3$ triangles, while the batch being rejected has a significantly lower average number of triangles, {namely} $1.2$. Despite a well estimated edge density, this batch produces a low $p$-value. {This batch, identified by \AgraSSt{} as less reliable, may not be very suitable for downstream tasks and it may be better to generate another batch instead.}
%
{In contrast, t}he batch with $p$-value $0.75$ has $2.28$ triangles on average while the batch with $p$-value $0.37$ has $2.04$ triangles on average;  {these averages are closer to the observed number of triangles}.

\begin{figure}[htp!]
    \centering
    \vspace{.2cm}
    \includegraphics[width=.87\textwidth]{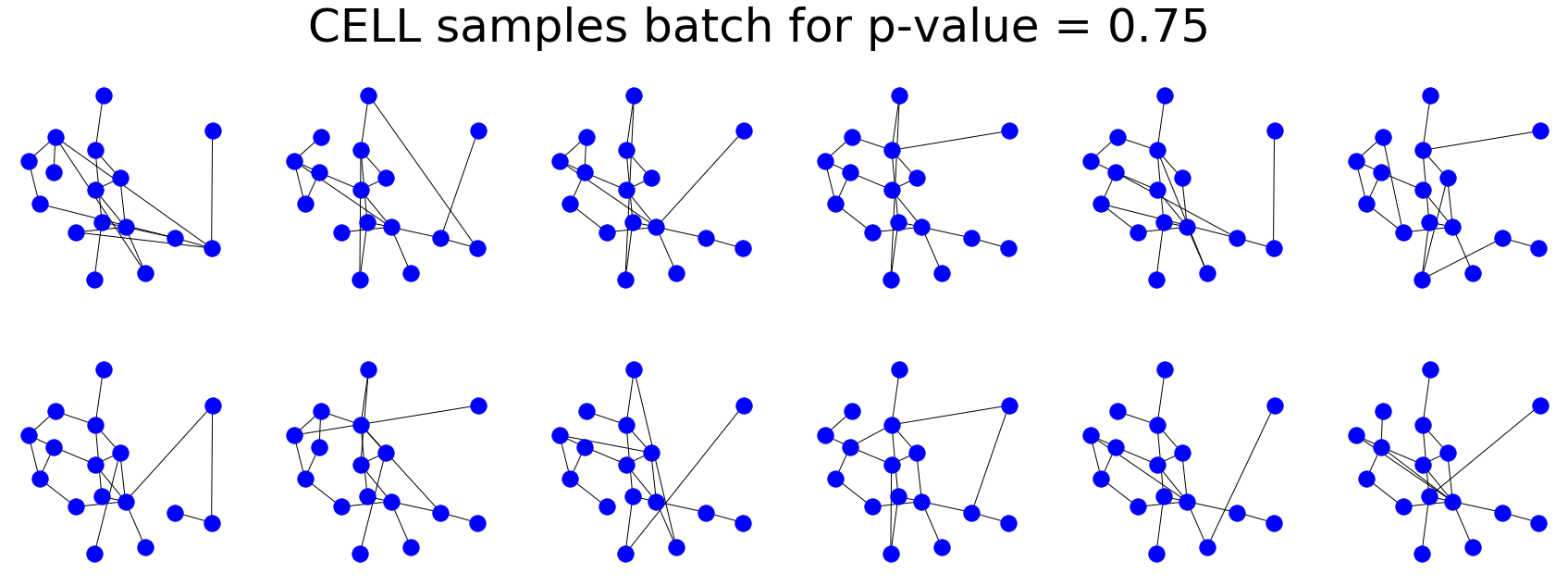} 
    
    \vspace{0.8cm}
    \includegraphics[width=.87\textwidth]{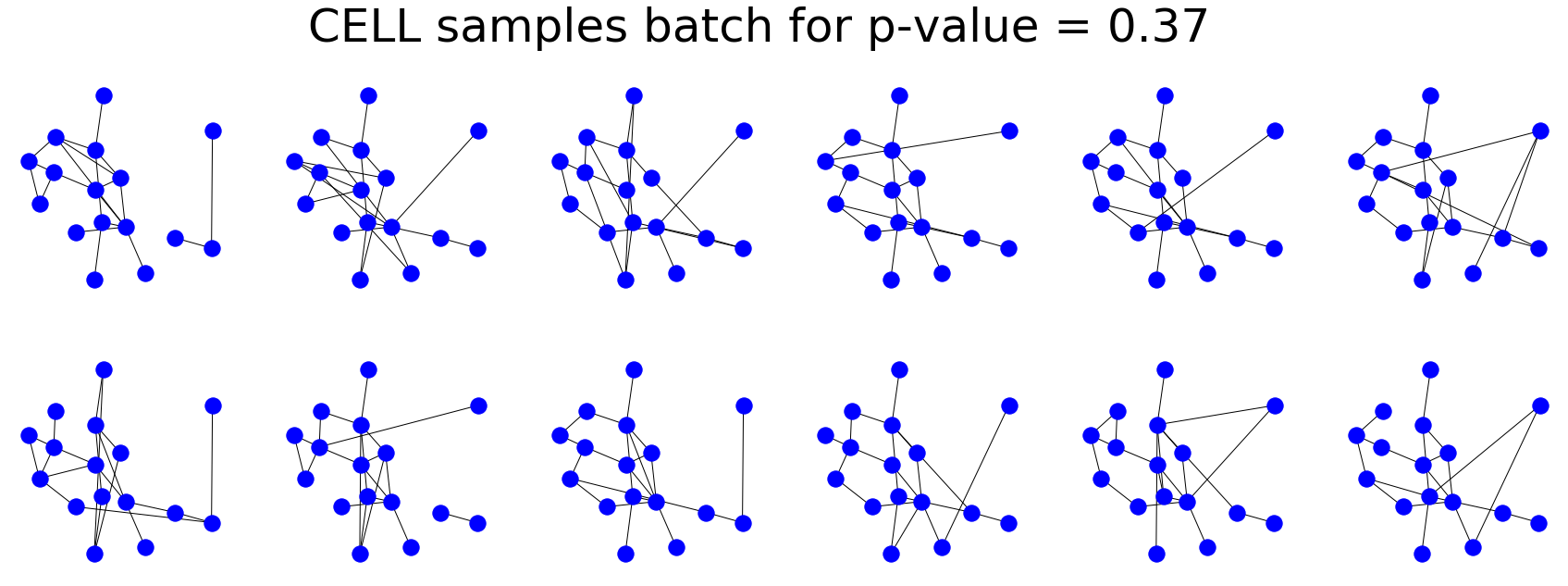} 
    
    \vspace{0.8cm}
    \includegraphics[width=.87\textwidth]{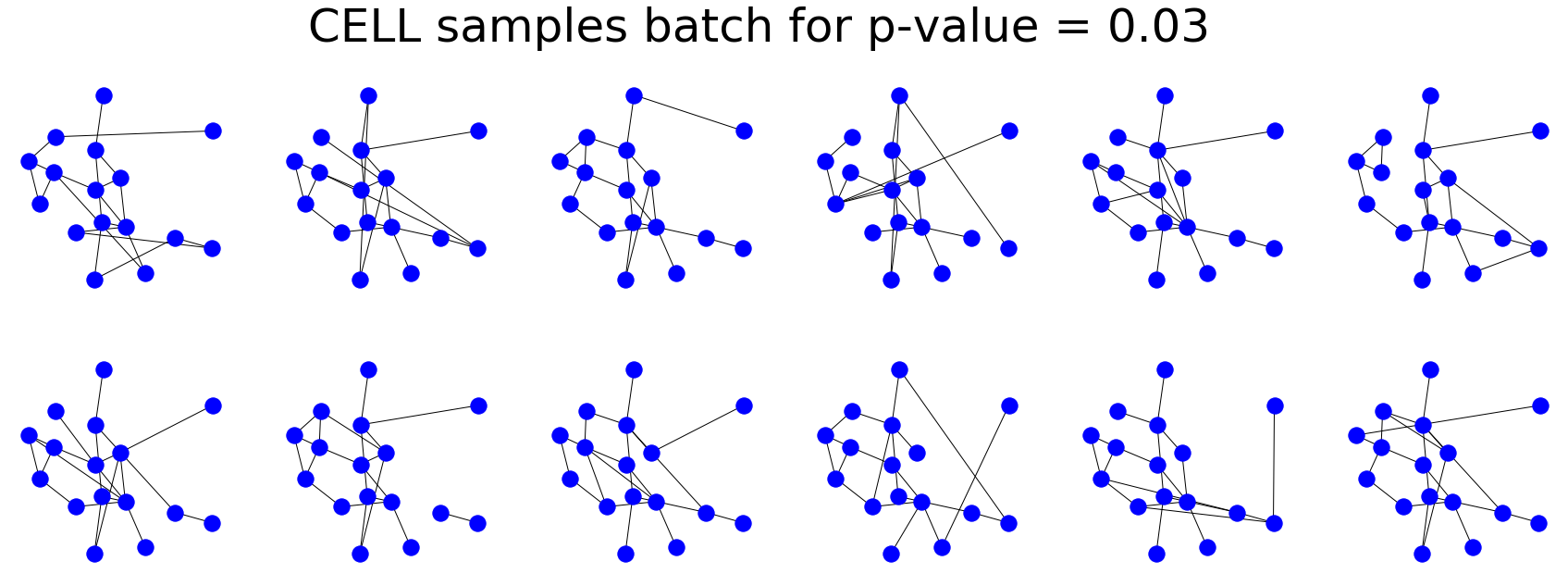} 
\caption{Samples from small size batches generated from CELL 
    trained on the Florentine marriage network, {with \AgraSSt{} $p$-values. The first two batches would be deemed suitable by \AgraSSt{}, while \AgraSSt{} would not accept the third sample at the 5\% significance level.}}
    \label{fig:florentine_pvalue}
        \label{fig:batches} 
\end{figure}

\subsection{Experiments with other network statistics} \label{app:additional}

{{AgraSSt can} 
incorporate any user-defined network statistics. 
{\cref{tab:additional} and \cref{tab:add_exp}  {show additional results in}
the settings of \cref{fig:e2st_statistics}  and \cref{tab:deep_model_validation},  respectively.} 
{As AgraSSt network statistics } {$t(x_{-s})$,} we {introduce} {D3, {which is} based on the}  multivariate statistics 
(edges((i,j)),deg(i),deg(j)), 
and 
{we  introduce}
Tri, 
{which is} 
based on the number of {common neighbours of $i$ and $j$}.
{T}he edge based AgraSSt {from} 
the main text {is added in grey} for comparison.}

\begin{table}[h!]
    \centering
\begin{tabular}{c|rrrrr}
\toprule
{perturbed $\beta_2$} &     -0.60 &     -0.40 &  -0.20 &     0.00 &   0.20 \\
\midrule
AgraSSt\_D3 
&  0.93 &  0.87 &    0.60 &  0.06 &    1.00 \\
AgraSSt\_{Tri} & 0.82 & 0.71 & 0.35 & 0.07 & 1.00\\
\midrule
{\color{gray}AgraSSt (main)} & {\color{gray}{0.95}} & {\color{gray}{0.89 }}& {\color{gray}{0.68}}& {\color{gray}{0.04}} & {\color{gray}{1.00}} \\
\bottomrule
\end{tabular}
\vspace{0.4cm}
\caption{Rejection Rate for the setting in \cref{fig:e2st_statistics}.}
    \label{tab:additional}
\end{table}

\begin{table}[h!]
\small
    \centering
\begin{tabular}{c|ccc|c}
        \toprule
        {Models} & GraphRNN 
        & NetGAN  &
        CELL & 
        MC \\
        \midrule
        AgraSSt\_{D3} & 0.31 
        & 0.66 &
        0.10 & 
        0.03 \\
        AgraSSt\_{Tri} & 0.28  
        & 0.32 &
        0.12 & 0.06
         \\
         \midrule
{\color{gray}{    AgraSSt (main)}}  & {\color{gray}{0.42  }}
        & {\color{gray}{0.81 }}&
       {\color{gray}{ 0.05}} & {\color{gray}{0.04}}
         \\
        \bottomrule
    \end{tabular}
\vspace{0.4cm}
\caption{Rejection Rate for the setting in Table 1.}\label{tab:add_exp}
\end{table}

{In {the} Florentine network example, D3 has $p$-values
0.04 for GraphRNN, 0.11 for NetGAN, and 0.74 for CELL.
Tri has $p$-values 0.02 for GraphRNN, 0.01 for NetGAN, and 0.12 for CELL.
{Overall, the results are mainly comparable to using AgraSSt based on the number of edges, although Tri rejects NetGAN for the Florentine marriage network, thus picking up on NetGAN struggling to reproduce local clustering.} }

\subsection{Additional discussions on distance-based test statistics}\label{app:TV}
{A classical approach for}
goodness-of-fit testing {in ERGMs is the graphical test by}
\citet{hunter2008goodness}. The idea is to simulate {sample graphs under the null distribution} statistics and  {create box plots of}  {some} relevant network statistics; {add to these plots the  network statistics in the observed network, as a solid line} for comparison, which is illustrated in Figure.\ref{fig:graphical_test_demo}. The box plot is used to check whether the observed network is ``very different'' from the simulated null samples. {This graphical test procedure can be translated into Monte Carlo tests.}
It is natural to adapt such procedure to implicit models {from which} samples can be obtained. Figure.\ref{fig:graphical_test_demo} plots {standard network statistics from \citet{hunter2008goodness} for} samples from a  fitted $G(n,p)$ generator {(ER Approximate)} and a learned GraphRNN generator of the Florentine marriage network {described in more detail in \cref{app:karate}}. 
The bold black line indicates the distribution of statistics for the Florentine marriage network.

The distribution of network statistics is then quantified via Total Variation (TV) distance \citep{xu2021stein}, based on which  goodness-of-fit testing with $p$-values can be conducted. {We find that while the fitted ER generator shows a reasonable fit for all summary statistics, the GraphRNN generator does not match the Florentine marriage network very well for dyad-wise shared partners and the triad census.} 
\begin{figure*}[htp!]
\begin{center}
\includegraphics[width=0.795\textwidth]{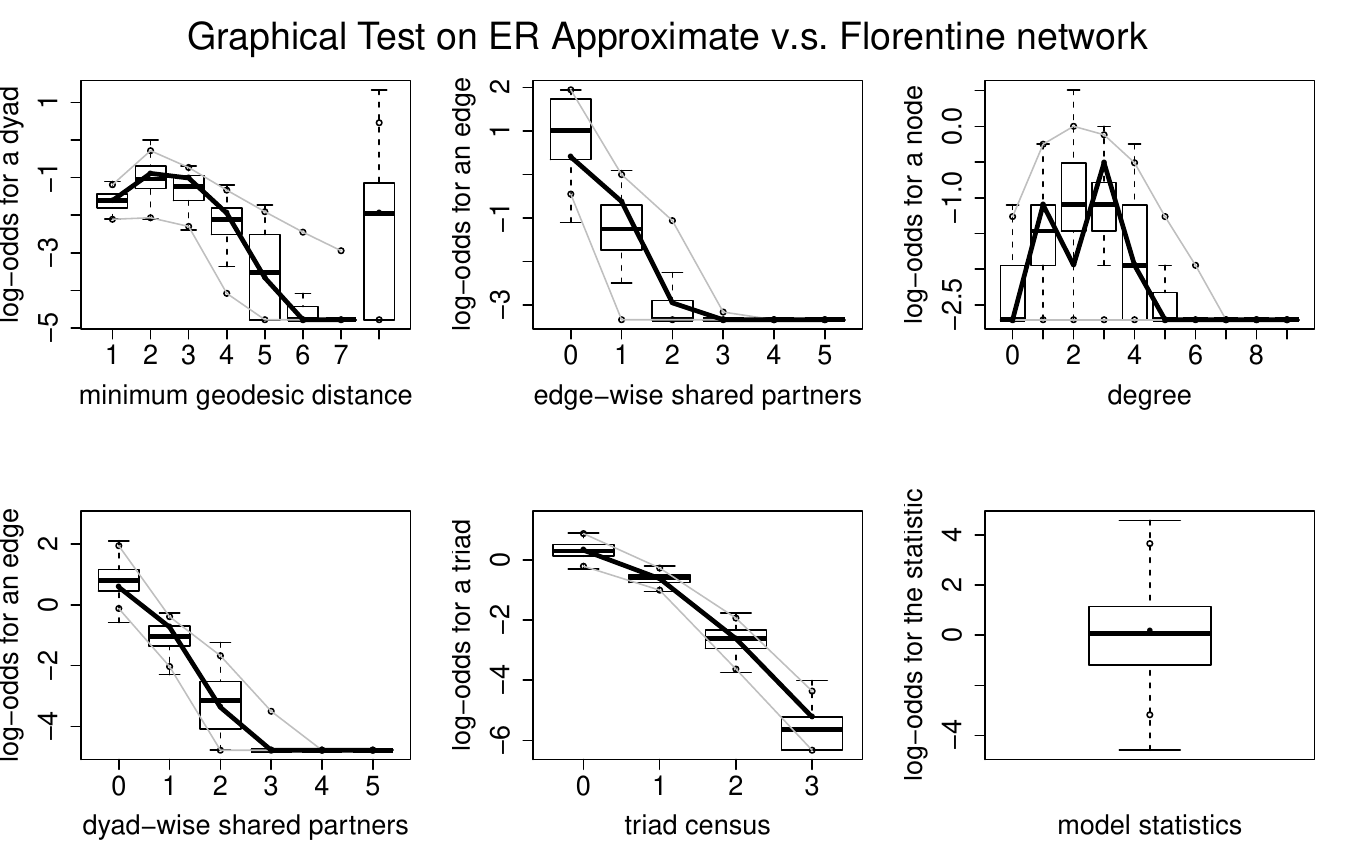}

\vspace{0.8cm}
\includegraphics[width=0.795\textwidth]{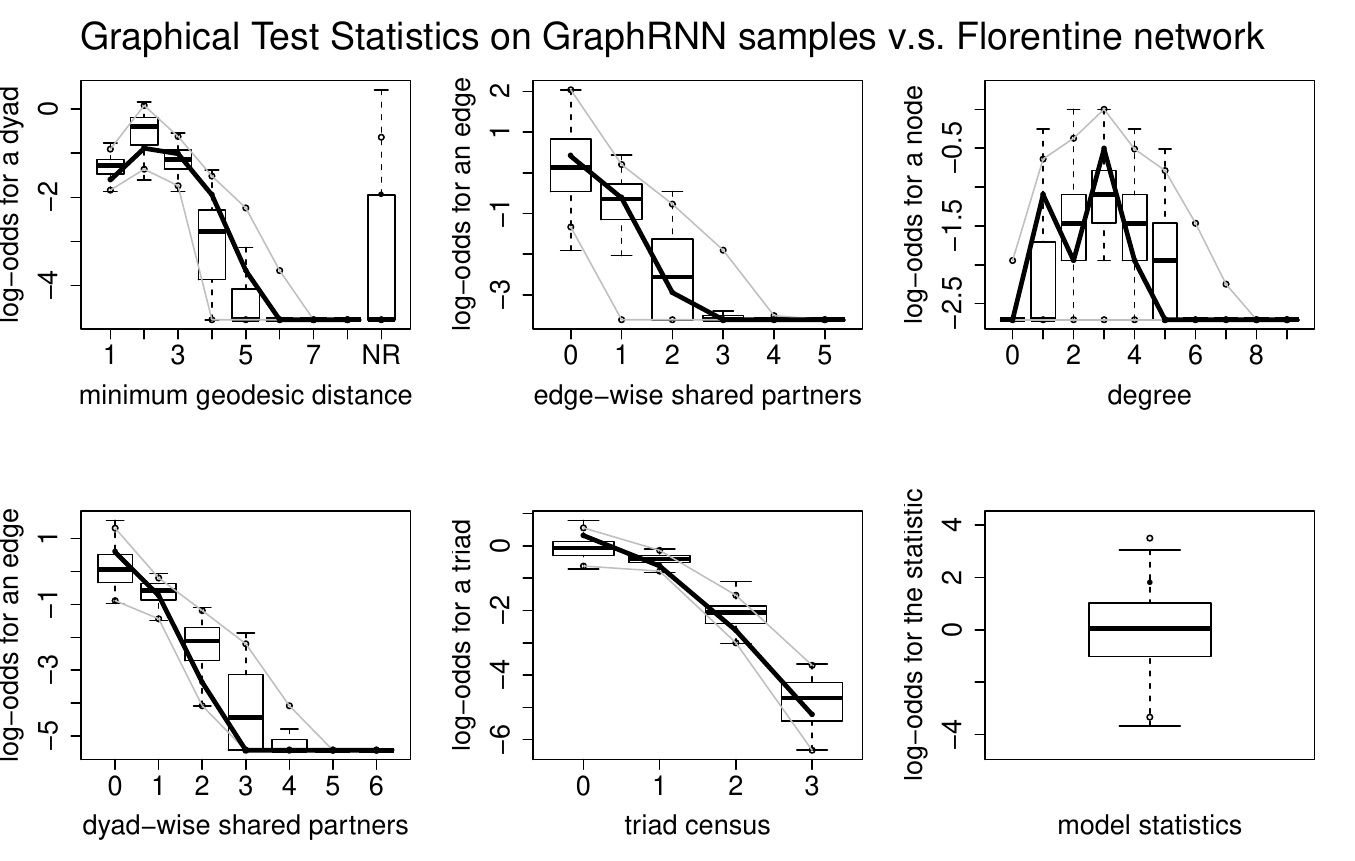}
	\caption{Graphical test illustrations on samples from generators learned on the Florentine marriage network}\label{fig:graphical_test_demo}
	\end{center}
\end{figure*}



\subsection{Efficiency results}\label{app:exp_comp}

Table.\ref{tab:runtime} presents the computational runtime (RT) {and} the test construction time (CT) {for \AgraSSt{} and its comparison methods from Section \ref{subsec:related}} {with the simulation setup as in \cref{subsec:simresults}}. {As a measure of accuracy,} 
the variance (Var) of the simulated (or estimated) test statistics {under the null distribution is also included}.

The parameter estimation in \textbf{Param} depends on a computationally efficient method  which is based on MPLE \citep{schmid2017exponential} in Eq.\eqref{eq:mple}. \textbf{AgraSSt} takes longer to compute 
mainly due to the computation of graph kernels, e.g.  Weisfeiler-Lehman kernel \citep{shervashidze2011weisfeiler}.
{We} 
note that for implicit models, the estimation step {in \AgraSSt{} } relies on generating samples from the model so that the 
the computational advantage\footnote{These results on gKSS are shown in Supplementary Material D in \citet{xu2021stein}.} of the Stein based test over
graphical goodness-of-fit tests\footnote{The graphical test \citep{hunter2008goodness} is computed based on generating a large amount of samples from the null distribution.} reduces {compared to \gKSS{}}. 
\textbf{MD{deg}} is computationally expensive due to the estimation of an inverse covariance matrix.
{While providing} fast computation and estimation, \textbf{Deg} and \textbf{Param} sacrifice test power {through a} large variance  
of the test statistics. Estimating the full degree distribution, {the total variation distance method}  \textbf{TV\_deg}, {based only on degrees,} 
is competitive with {\AgraSSt{}}; {we recall that in our simulation results from  \cref{subsec:simresults}  \textbf{TV\_deg} was less powerful than {\AgraSSt{}}.} 
{Here} \textbf{MD{deg}} {is outperformed by the other test statistics.}

\begin{table}[htp!]
    \centering
    \begin{tabular}{c|ccccc}
        \toprule
       {}& AgraSSt &  
        Deg & 
        Param & 
        MDdeg & 
        TV\_deg \\
        \midrule
        RT(s)  & 0.141 & 0.0006 & 0.014 & 0.831 & 0.002 \\
        CT(s)  & 28.656 & 0.277 & 2.963 & 162.912 & 0.555\\
        \midrule
        Var & 0.23 & 8.38 & 1.43 & 15.84 & 0.28   \\
        \bottomrule
    \end{tabular}
    \vspace{0.5cm}
    \caption{Computational efficiencies and uncertainty in estimates. RT: runtime for one test; CT: construction time for the test class, including generating 500 samples for relevant estimation and 200 samples for simulating {from the} null distribution; {Var}: the estimated variance {under the} simulated null distribution. Both RT and CT are in seconds.}
    \label{tab:runtime}
\end{table}





\subsection{Additional implementation details} \label{app:implementation} 


We note that training NetGAN \citep{bojchevski2018netgan}  with the Florentine and with the Karate Club network may encounter some generator instability and hence early stopping can be useful. 
Without early stopping, the training loss for the generator increases during training, although it should be decreasing. 
Figure.\ref{fig:training_loss} shows the training loss on the generator in NetGAN as well as on the critic (or discriminator) in NetGAN. Figure.\ref{fig:training_loss} plots the loss every 200 training epochs. We can see from Figure.\ref{fig:training_netgan_flo} that the generator loss starts to be unstable and then increases after 50 points, i.e. 10,000 epochs.
Hence we use only 10,000 epochs for training. 


\begin{figure}[htp!]
    \centering
    \subfigure[Florentine network]{\includegraphics[width=0.3\textwidth]{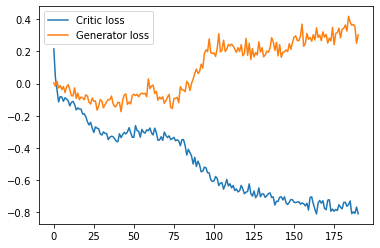}\label{fig:training_netgan_flo}}    \subfigure[Florentine network with early stopping]{\includegraphics[width=0.3\textwidth]{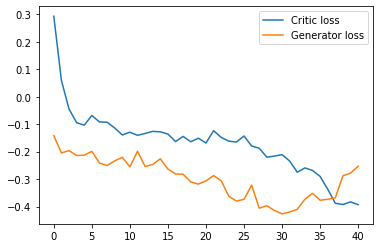}}    \subfigure[Karate Club network]{\includegraphics[width=0.3\textwidth]{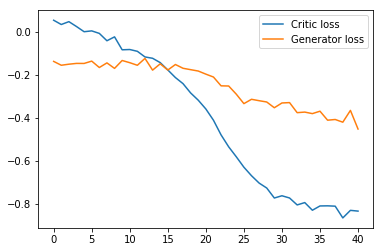}}
    \caption{Training loss (y-axis) for NetGAN \citep{bojchevski2018netgan}; plot{ted} against  every 200 training epochs (x-axis).
    }
    \label{fig:training_loss}
\end{figure}

\end{document}